\def\year{2020}\relax
\documentclass[letterpaper]{article} 
\usepackage{aaai20}  
\usepackage{times}  
\usepackage{helvet} 
\usepackage{courier}  
\usepackage[hyphens]{url}  
\usepackage{graphicx} 
\usepackage{graphicx}  
\usepackage{bm}
\usepackage{amssymb}
\usepackage{amsthm}
\usepackage{amsmath}
\usepackage[ruled,vlined]{algorithm2e}
\usepackage{booktabs}
\usepackage{multirow}
\usepackage{color,colortbl}
\definecolor{darkgreen}{rgb}{0, 0.5, 0}
\definecolor{red}{rgb}{1, 0, 0}
\definecolor{purple}{rgb}{0.5, 0, 0.5}
\usepackage{pxfonts}

\newcommand\ie{\textit{i.e.}}
\newcommand\st{\textit{s.t.}}
\newcommand\eg{\textit{e.g.}}

\newcommand\etc{\textit{etc.}}
\newcommand\wrt{\textit{w.r.t.}}

\newtheorem{definition}{Definition}

\usepackage[caption=false,font=footnotesize]{subfig}

\newcommand\algoabbr{TER}
\newcommand\titletext{Generalizable Meta-Heuristic based on Temporal Estimation of Rewards for Large Scale Blackbox Optimization}

\newtheorem{proposition}{Proposition}
\title{\titletext{}}

\author{Mingde Zhao$^{1,2,3}$, Hongwei Ge$^{3}$, Yi Lian$^{2}$, Kai Zhang$^{3}$\\
1. Mila; 2. McGill University; 3. Dalian University of Technology
}

\begin{document}
\maketitle
\begin{abstract}
The generalization abilities of heuristic optimizers may deteriorate with the increment of the search space dimensionality. To achieve generalized performance across Large Scale Blackbox Optimization (LSBO) tasks, it is possible to ensemble several heuristics and devise a meta-heuristic to control their initiation. This paper first proposes a methodology of transforming LSBO problems into online decision processes to maximize efficiency of resource utilization. Then, using the perspective of multi-armed bandits with non-stationary reward distributions, we propose a meta-heuristic based on Temporal Estimation of Rewards (\algoabbr{}) to address such decision process. \algoabbr{} uses a window for temporal credit assignment and Boltzmann exploration to balance the exploration-exploitation tradeoff. The prior-free \algoabbr{} generalizes across LSBO tasks with flexibility for different types of limited computational resources (\eg{} time, money, \etc{}) and is easy to be adapted to new tasks for its simplicity and easy interface for heuristic articulation. Tests on the benchmarks validate the problem formulation and suggest significant effectiveness: when \algoabbr{} is articulated with three heuristics, competitive performance is reported across different sets of benchmark problems with search dimensions up to $10000$.
\end{abstract}

\section{Introduction}
Optimization problems with partial information, mappings too complex to capture are regarded as ``blackbox'' optimization problems \cite{cantin2018funneled}. With minimal assumptions and no requirements for domain-specific knowledge, heuristic methods show promising performance in these problems \cite{ma2018survey,zhan2018dis,cheng2018incremental,cao2018multimodal}. Large Scale Blackbox Optimization (LSBO) problems focus on optimizing deterministic scalar functions with high-dimensional search space, constrained by limited number of computational resources, \eg{} the number of function evaluations, time and money, \etc{}. Under this setting however, the performance of heuristics deteriorate as most of them are not scalable to higher dimensional search spaces \cite{SUN2018Rec,ge2017cooperative,loshchilov2018limited}.
\par
Efforts have been put into devising heuristics with more scalable performance \cite{li2018fast,Al-Dujail2017EMB}. However, empirical evidence suggests that their unsatisfactory generalization abilities: their capabilities inevitably vary when applied on problems with different characteristics. In the LSBO setting, it is nearly impossible to understand the nature of a given problem without prior knowledge. This no-free-lunch phenomenon directs researchers to the field of meta-heuristics, which are higher-level procedures to initiate heuristics using the feedback during the optimization process. With meta-heuristics, users can focus on articulating powerful heuristics to obtain satisfactory performance.
\par
However, the existing meta-heuristic methods are still with huge space for improvement for the following concerns: 1) Recognition of the objective: the LSBO problems can all be summarized as given limited number of computational resources (\eg{} the number of evaluations for the blackbox objective functions), achieve the lowest possible error by searching the best solutions. The true objective of these tasks is the efficiency of resource utilization, which most of the existing methods fail to recognize. In the literature of existing methods for meta-heuristics, many surrogate objectives for resource utilization, \eg{} number of individuals that gained better fitness during evolution, however there is no direct connection between these surrogates and the true objective of efficiency; 2) Overfitting: for better performance on benchmark functions, a large portion of existing meta-heuristics are designed with components specially effective for the corresponding test cases. The prior knowledge for designing ``overfitting'' meta-heuristics on those tasks is of no help for deeper investigation. The meta-heuristics of this kind are expected to lose their effectiveness once transferred to other tasks. For example, SHADE-ILS \cite{molina2018iterative}, the current state-of-the-art method, calls a global search heuristic that is particularly effective for CEC'2013 benchmark problems in every iteration. It can outperform most of the existing methods on those tasks by a significant margin, yet only achieving intermediate performance when applied on CEC'2008 and CEC'2010 benchmark functions. Similarly, MTS \cite{tseng2008multiple}, which was designed with the same methodology, achieves matchless performance on the CEC'2008 benchmark problems however performs badly on the later benchmark problems; 3) Generalization: Some meta-heuristics are with too many hyperparameters to tune for the tasks, making it hard to generalize and transfer.

The main contributions of this paper lie in:

\begin{itemize}
\setlength{\itemsep}{0pt}
\setlength{\parsep}{0pt}
\setlength{\parskip}{0pt}
\item
We transform the LSBO problem setting into a decision process, in which the actions are the initiations of heuristics and the rewards are the improvement of best known solution during the initiations. In this new setting, the objective of the LSBO problem is converted to maximizing the expected cumulative return. This principled framework is not only meaningful for theoretical purposes but also beneficial for application scenarios.
\item
We propose a simplified view of treating such decision process as multi-armed bandit problems and propose a simple algorithm of using a window for temporal credit assignment, which enables local estimation of the non-stationary reward distributions. With this, we use a Boltzmann exploration (softmax) method to balance exploration and exploitation.
\item
We analyze the complexity and the behavior of the proposed algorithm, and give principled guidelines for practical use.
\end{itemize}

\section{Preliminaries}
Most of the existing meta-heuristics suffer from the ``overfitting'' problem: the meta-heuristics are often proposed together with several specific articulated heuristics and it is either impossible or very costly to re-articulate the meta-heuristics with new heuristics. The difficulties are caused by the fact that researchers introduce strong prior knowledge and couplings into the designed systems. Meta-heuristic is however emerged with the expectation to automatically adapt to the differences in diverse objectives to optimize. For such goal, a prior-free meta-heuristic with easy interfaces for heuristic articulation is desired. To reach it, a principled framework that could capture the key features of optimization is needed. In the following parts of this section, we introduce a perspective to transform the optimization problem (with meta-heuristics) into a decision process and then discuss its characteristics.

\subsection{Notations \& Definitions}
Before stating the ideas, some prerequisite notations and definitions are to be introduced:

\begin{definition}[Best Fitness]
Given a scalar objective function $f$ and some $t$ that represents the amount of consumed computational resources, \eg{} time elapsed or the number of used function evaluation, the best known objective value (minimum for minimization or maximum for maximization) queried via function evaluation from the start of the optimization process until $t$ is called \textbf{best fitness} until $t$.
\end{definition}

\begin{definition}[Improvement \& Efficiency]
Given $f$ and $t_1, t_2$ of computational resource consumption, where $t_1 < t_2$, the difference between the best fitness $y_2$ until $t_2$ and the best fitness $y_1$ until $t_1$ is called the \textbf{improvement of global best fitness} (or \textbf{improvement}) during $(t_1, t_2]$. The \textbf{efficiency} during such period is defined as the fraction of improvement and the resource consumption, \ie{} $(y_2 - y_1) / (t_2 - t_1)$.
\end{definition}

\begin{definition}[Efficiency of Actions]
Given $f$, the fraction of the improvement (from $y_1$ to $y_2$) and the consumption of resources (from $t_1$ to $t_2$) after taking an action $a$ is defined as the \textbf{efficiency of action} $a$ (or \textbf{action efficiency}), denoted as $\mathcal{E}(y_1, y_2, t_1, t_2, a)$.
\end{definition}

The notions of best fitness and improvement are algorithm agnostic. By cutting the resource consumption into equal pieces, we can achieve a piece-wise approximation of the real-time efficiency along the optimization process.

\begin{definition}[Expected Overall Improvement]
Given function $f$ to be optimized, and an initial distribution $d_0$ of the global best fitness, the \textbf{expected overall improvement} $\mathbb{E}_{\mathcal{A}, y_0 \sim d_0}\left[\Delta_{0:T}\right]$ achieved by algorithm $\mathcal{A}$ given all the computational resources $T$ provided is defined as
$$\mathbb{E}_{\mathcal{A}, y_0 \sim d_0}\left[\Delta_{0:T}\right] = \mathbb{E}_{\mathcal{A}, y_0 \sim d_0}\left[ y{*}(\mathcal{A}, y_0, T) - y_0 \right]$$
where $y{*}(\mathcal{A}, y_0, T)$ is the best fitness found by $\mathcal{A}$ from initialization $t_0$ until $t=T$.
\end{definition}

\begin{definition}[Expected Overall Efficiency]
Given function $f$ to be optimized, and a initial distribution $d_0$ of the global best fitness, the \textbf{expected overall efficiency} $\mathbb{E}_{\mathcal{A}, y_0 \sim d_0}\left[\mathcal{E}_{0:T}\right]$ achieved by algorithm $\mathcal{A}$ given all the computational resources $T$ provided is defined as
$$\mathbb{E}_{\mathcal{A}, y_0 \sim d_0}\left[\mathcal{E}_{0:T}\right] = \mathbb{E}_{\mathcal{A}, y_0 \sim d_0}\left[ y{*}(\mathcal{A}, y_0, T) - y_0 \right] / T$$
\end{definition}

The notions of improvement and efficiency is translation invariant, \ie{} they do not change if the search landscape is shifted up or down.

\begin{proposition}[Efficiency Maximization is Optimization]
Given fixed initialization scheme and the same computational resources, maximizing expected overall improvement (efficiency) is equivalent to optimizing the objective function itself.
\end{proposition}

\begin{proof}
Trivial.
\end{proof}

\subsection{Formulation}
Whether using meta-heuristics, the pursuit of resource-constrained optimization should be maximizing the expected overall efficiency, in a sense that with limited computational resources, we should improve the best fitness by the most. Such recognition gives us a perspective of building meta-heuristics: we can treat initiation of the articulated heuristics as actions and taking the actions will result in improvements of the best fitness at the cost of some resources. The fraction of these two terms yields an ``action efficiency'' that is closely connected to the ``efficiency'' that we pursue ultimately, which is the improvement upon the initial objective value divided by the total number of computational resources given: the integral of the action efficiency yields the overall efficiency.

Using the definitions, the goal is formalized as:
$$\textit{Find a sequence }\{a^{(i)}\}, i \in \{1, \dots, n\}\textit{, to maximize}$$
\begin{align}
\mathbb{E}\left[\sum_{i = 1}^{n}{(t_2^{(i)} - t_1^{(i)}) \mathcal{E}(f_1^{(i)}, f_2^{(i)}, t_1^{(i)}, t_2^{(i)}, a^{(i)})}\right], \text{\st{}} \sum_{i = 1}^{n}{(t_1^{(i)} - t_2^{(i)})} \leq T
\end{align}
where $T$ is the total amount of given resources.

With this formulation, the meta-heuristic problem is transformed into online decision problems to maximize the overall efficiency of taking actions. We can fit the online decision process with the goal to maximize the expected cumulative rewards.

Devising meta-heuristics under this formulation requires the heuristics to be decoupled, \ie{} each action does not affect the behavior of others. At the cost of potential information loss during consecutive initiations of heuristics, the decoupling gives lots of advantages by enabling the meta-heuristics to generalize to new tasks by granting them the flexibility of articulating any appropriate heuristics. The flexibility gives us the freedom to use the best configuration (\wrt{} resource allocation for each initiation) of the articulated heuristics, which is crucial of the effective heuristics are sensitive to the resource allocation. Moreover, the decoupling grants the decision process with semi-Markovian properties, which may be systematically investigated in the future research.

Unfortunately, because of the expensive nature of the problem, episodes of training is prohibited and thus the problem cannot be solved using most of the existing methods for semi-MDPs.

\section{\algoabbr{}}

The decision we seek is under arguably the most complex scenarios - partially observable semi-Markov decision processes. Also, for LSBO problems, we can only run for one episode as there is no training allowed. However, if we reasonably assume that the reward distribution of taking an action changes slowly with the transition of the states, \ie{} the reward distribution of each action changes gradually overtime, we can approximate the decision problem with a multi-armed bandits setting (with non-stationary rewards). In such setting, we have to find an effective way to do temporal credit assignment \st{} the recent rewards can be used to well-approximate the rewards for the recent future.

With the reached formulation, in this section, we propose an algorithm to address the decision process.

Typical temporal credit assignment techniques in multi-armed bandits include exponential weighting, change point detection, \etc{}. In this paper, we turn to the simplest of its kind, a fixed-size window for lower variance behavior and better generalization. The decisions of actions only consider the pieces of performance of the recent initiations of heuristics inside the window, \ie{} we use only the recent rewards to estimate the recent future.

To implement such idea, we have a memory queue (first-in-first-out) of size $w$ to buffer all the efficiency records of the heuristic initiations. When an initiation of some heuristic is finished, we first linearly normalize the efficiency values into the interval $[0, 1]$ and use the mean of the normalized efficiency records of each action to approximate the mean of the reward distribution in the recent future. The normalization helps to regularize the behavior of the bandit algorithms and give us a bound of the greediness when used with soft decision methods.

\subsection{Boltzmann Exploration with Strict Exploration}
We use a modified Boltzmann exploration (softmax) method to handle the non-stationary multi-armed bandits problem: we softmax the mean of the normalized rewards for each action within the sliding window to obtain the probability distribution of selecting each action and sample the actions accordingly.

The greediness of softmax is controlled by a hyperparameter $\tau$, with which we divide the energy values, \ie{} the means, and take the exponential, as follows.
$$\text{softmax}(\bm{x}, \tau) \equiv \frac{e^{\bm{x} / \tau}}{\bm{1}^T e^{\bm{x} / \tau}}$$
The higher the value of $\tau$, the less greedy the output distribution is \wrt{} the energy values in $\bm{x}$. $\tau$ can also be set to be time-variant or more complex. In this paper, we treat it as a fixed hyperparameter.

Boltzmann exploration alone is problematic for the case in which there are no records of an action within the window. Since softmax decision requires at least one mean value for each action. This can be fixed with a simple rule: we define the mean of an empty set to be infinite. This means if no record of an action is found, such action will be taken directly. This additional rule not only ensures exploration across the whole optimization process no matter how bad the hyperparameters are chosen but us also beneficial for exploration which must be guaranteed in such a non-stationary setting: if a heuristic behaved badly before, it may behave well now.

The elaboration of our proposed meta-heuristic framework \algoabbr{} is now complete. Its pseudocode is presented in Algorithm \ref{algorithm:slider}.

\begin{algorithm}[t]
\label{algorithm:slider}
\caption{\algoabbr{}-based LSBO}
\LinesNumbered
\KwIn{$f$ (blackbox function to be optimized), $A$ (set of heuristics), $T$ (total number of computational resources), $\tau$ (greediness for softmax), $w$ (size of temporal window), $\bm{x}_{best}$ (initial $\bm{x}_{best}$), $y_{best}$ (function value of initial $\bm{x}_{best}$)}
\KwOut{$\bm{x}_{best}$ (best known solution), $y_{best}$ (function value of $\bm{x}_{best}$)}

$Q \gets queue()$; \textcolor{darkgreen}{//initialize an empty queue}\\

$t \gets 0$; \textcolor{darkgreen}{//initialize counter for consumed resources}\\

\While{$t < T$}{
\textcolor{darkgreen}{//normalize efficiency values in $Q$ to $[0, 1]$}\\
$Q_{temp} \gets \text{normalize}(Q, 0, 1)$;\\

\textcolor{darkgreen}{//sample action from distribution formed by softmax}\\
$\bm{x} \gets \bm{0}_{|A| \times 1}$\\
\For{$i \in \{ 1, \dots, |A|\}$}{
    $x_i \gets \text{mean}(\text{records of $a_i$ in $Q_{temp}$})$; \textcolor{darkgreen}{//$\infty$ if no existing records of $a_i$}\\
}
$a \sim \text{softmax}(\bm{x}, \tau)$; \textcolor{darkgreen}{//sample action}\\

$[\bm{x}_{best}, y_{best}, \Delta t] \gets \text{apply\_on}(a, f, \bm{x}_{best}, y_{best})$; \textcolor{darkgreen}{//apply heuristic $a$ on $f$, update $\bm{x}_{best}, y_{best}$, get the consumption $\Delta t$}\\

$Q.add(\langle \bm{x}_{best}, y_{best}, \Delta t \rangle)$; \textcolor{darkgreen}{//add record into $Q$}\\

\If{$|Q| > w$}{
    $Q.pop()$; \textcolor{darkgreen}{//only maintain $Q$ with size $w$}\\
}
$t \gets t + \Delta t$;
}

\end{algorithm}

\section{Analyses}
\subsection{Computational Complexity}
For runtime complexity, the mean and softmax operate at the level $\mathcal{O}(w|A|)$, which is trivial compared to those of the heuristics. For space complexity, since the historical records outside the credit assignment window can be safely discarded, a trivial one at most $\mathcal{O}(w)$ can also be reached.

\subsection{Probability Bounds for Behavior}

\algoabbr{} has two hyperparameters, $w$ (length of the temporal window) and $\tau$ (temperature coefficient for softmax). Suppose that there are $|A|$ actions, we have the following proposition for the behavior of \algoabbr{}:

\begin{proposition}[Exploration Bounds for \algoabbr{}]\label{prop:exploitation}
Suppose every action is corresponded with at least one efficiency record within the window of size $w$ and $\tau$ is the softmax hyperparameter, the probability of taking an action with not-the-highest mean normalized record (exploration) in the fragment of length $w$ satisfies the bounds
\begin{equation}
p_{\text{explore}} \in \left[\frac{e^{{1}/{\tau}}}{|A|-1+e^{{1}/{\tau}}}, \frac{|A|-1}{(|A|-1)\cdot e^{{1}/{\tau}}+e^{\frac{w-|A|}{\tau(w-|A|+1)}}}\right)
\end{equation}
\end{proposition}

\begin{proof}
See Appendix.
\end{proof}
The proposition gives a measurement of behavior for \algoabbr{}. We can use the bounds inversely to do efficient hyperparameter search, for which we offer details in the Appendix.

\section{Related Works}
There are different ideas of initiating the articulated heuristics. Some meta-heuristics initiate all the articulated heuristics in a certain order alternately. The primitive algorithms of this kind enable fixed resources for each heuristic \cite{rohler2015minimum}. There are also ideas of modifying the resource allocation for each heuristic dynamically. In \cite{tseng2008multiple}, the authors proposed to use some small amount of resources to test-run the articulated heuristics and run the best-performing one with large amount of resources. MOS \cite{latorre2012multiple} runs each of its articulated heuristics within each iteration and dynamically adjusts the resources for each heuristic according to their performance. The adjustment of the resources, or more accurately the change of the configuration for the articulated heuristics, is problematic for the fact that heuristics often require sufficient iterations to show their full capacities for the current state of optimization \cite{hansen2014principled}. This kind of allocation strategy is likely to undermine the effectiveness of heuristics, which leads to the loss of many potentially useful heuristic articulation choices; Some meta-heuristics samples the initiation of articulated heuristics according to some distribution constructed by the historical performance. In \cite{ye2014hybrid}, the authors proposed a greedy method, which simply switches from the two articulated heuristics when one of them has poor performance. Yet, the greedy exploitation of heuristic performance without in-time exploration may lead to premature convergence. In \cite{molina2015iterative}, the algorithm uses this idea for initiating some of the articulated heuristics. However, there are also some heuristic to be initiated in every main loop, which are specifically effective for the benchmark functions it was tested on. Also, since it does not use the efficiency as the evaluation criteria of temporal performance, the resources for initiating each heuristic is strictly fixed to be the same.

The second perspective is about the evaluation criteria for the performance of articulated heuristics. Most of the existing methods based on evolutionary strategies use the number of improved offsprings (which, after crossover and mutation, improved their fitness values) \cite{tseng2008multiple,latorre2012multiple}. Some also take into consideration the magnitude of change. Efficiency, however, have hardly been investigated in the literature.

\begin{table*}[!ht]
\setlength{\tabcolsep}{1.5pt}
\scriptsize
\renewcommand\arraystretch{0.8}
  \centering
  \caption{Baseline Comparison Results}
    \begin{tabular}{cccccccccccccccccc}
    \toprule
    \multirow{2}[1]{*}{LSO13} & \multicolumn{2}{c}{BEK} & \multicolumn{3}{c}{\textit{\algoabbr{}}} & \multicolumn{3}{c}{RAN} & \multicolumn{3}{c}{LS} & \multicolumn{3}{c}{CC} & \multicolumn{3}{c}{GS} \\
          & mean  & std   & \textit{mean} & \textit{std} & \textit{sim} & mean  & std   & sim   & mean  & std   & sim   & mean  & std   & sim   & mean  & std   & sim \\
    \toprule
	$f_{1}$    & \cellcolor[rgb]{ .388,  .745,  .482}0.00e0 & 0.00e0 & \cellcolor[rgb]{ .388,  .745,  .482}\textit{0.00e0} & \textit{0.00e0} & \textit{$\sim$} & \cellcolor[rgb]{ 1,  .922,  .518}4.41e-10 & 1.36e-9 & $\sim$     & \cellcolor[rgb]{ .388,  .745,  .482}0.00e0 & 0.00e0 & $\sim$     & \cellcolor[rgb]{ 1,  .922,  .518}4.77e-4 & 1.32e-3 & $\sim$     & \cellcolor[rgb]{ .973,  .412,  .42}6.85e5 & 1.01e5 & $\sim$ \\
    $f_{2}$    & \cellcolor[rgb]{ .388,  .745,  .482}2.25e-2 & 1.02e-2 & \cellcolor[rgb]{ .427,  .757,  .482}\textit{8.69e0} & \textit{2.78e0} & \textit{78.13\%} & \cellcolor[rgb]{ 1,  .918,  .518}2.44e2 & 3.83e1 & 31.56\% & \cellcolor[rgb]{ 1,  .902,  .514}5.54e2 & 4.02e1 & 23.87\% & \cellcolor[rgb]{ .396,  .745,  .482}2.40e0 & 1.50e0 & 92.25\% & \cellcolor[rgb]{ .973,  .412,  .42}9.89e3 & 1.45e3 & 2.33\% \\
    $f_{3}$    & \cellcolor[rgb]{ .388,  .745,  .482}4.85e-14 & 2.05e-14 & \cellcolor[rgb]{ .98,  .914,  .514}\textit{9.83e-13} & \textit{5.27e-14} & \textit{69.13\%} & \cellcolor[rgb]{ .929,  .902,  .514}9.06e-13 & 6.55e-14 & 28.57\% & \cellcolor[rgb]{ 1,  .922,  .518}1.04e-12 & 6.51e-14 & 71.43\% & \cellcolor[rgb]{ .996,  .831,  .502}1.19e0 & 1.11e-1 & 31.43\% & \cellcolor[rgb]{ .973,  .412,  .42}6.63e0 & 4.41e-1 & 0.00\% \\
    $f_{4}$    & \cellcolor[rgb]{ .388,  .745,  .482}5.45e8 & 1.45e8 & \cellcolor[rgb]{ .608,  .808,  .494}\textit{6.98e8} & \textit{2.51e8} & \textit{51.43\%} & \cellcolor[rgb]{ 1,  .922,  .518}1.23e9 & 3.90e8 & 16.47\% & \cellcolor[rgb]{ .996,  .847,  .506}1.86e10 & 9.98e9 & 5.88\% & \cellcolor[rgb]{ .973,  .412,  .42}1.22e11 & 6.65e10 & 12.25\% & \cellcolor[rgb]{ .518,  .78,  .486}6.34e8 & 1.00e8 & 68.24\% \\
    $f_{5}$    & \cellcolor[rgb]{ .388,  .745,  .482}9.54e5 & 4.32e5 & \cellcolor[rgb]{ 1,  .918,  .518}\textit{2.68e6} & \textit{4.38e5} & \textit{30.34\%} & \cellcolor[rgb]{ .969,  .91,  .514}2.52e6 & 9.94e5 & 42.15\% & \cellcolor[rgb]{ .976,  .431,  .424}1.23e7 & 2.93e6 & 11.05\% & \cellcolor[rgb]{ .973,  .412,  .42}1.26e7 & 3.42e6 & 20.01\% & \cellcolor[rgb]{ .847,  .878,  .506}2.20e6 & 6.93e5 & 65.98\% \\
    $f_{6}$    & \cellcolor[rgb]{ .388,  .745,  .482}3.93e4 & 3.55e4 & \cellcolor[rgb]{ .631,  .812,  .494}\textit{4.44e4} & \textit{3.42e4} & \textit{44.51\%} & \cellcolor[rgb]{ 1,  .918,  .518}5.99e4 & 3.35e4 & 46.58\% & \cellcolor[rgb]{ .973,  .412,  .42}9.84e5 & 4.13e3 & 4.29\% & \cellcolor[rgb]{ .976,  .416,  .424}9.82e5 & 4.62e3 & 2.17\% & \cellcolor[rgb]{ .427,  .757,  .482}4.02e4 & 3.32e4 & 71.35\% \\
    $f_{7}$    & \cellcolor[rgb]{ .388,  .745,  .482}8.07e4 & 9.25e3 & \cellcolor[rgb]{ .459,  .765,  .486}\textit{1.58e5} & \textit{3.68e4} & \textit{44.25\%} & \cellcolor[rgb]{ .694,  .831,  .498}4.10e5 & 4.22e5 & 33.15\% & \cellcolor[rgb]{ 1,  .875,  .51}3.10e8 & 4.81e8 & 19.77\% & \cellcolor[rgb]{ .973,  .412,  .42}3.13e9 & 2.20e9 & 24.24\% & \cellcolor[rgb]{ 1,  .922,  .518}1.06e6 & 1.00e5 & 33.37\% \\
    $f_{8}$    & \cellcolor[rgb]{ .388,  .745,  .482}2.93e9 & 4.35e9 & \cellcolor[rgb]{ .478,  .769,  .486}\textit{1.24e11} & \textit{9.29e10} & \textit{51.25\%} & \cellcolor[rgb]{ 1,  .922,  .518}1.45e12 & 1.42e12 & 44.42\% & \cellcolor[rgb]{ 1,  .894,  .514}2.96e15 & 1.75e15 & 21.24\% & \cellcolor[rgb]{ .973,  .412,  .42}4.93e16 & 3.12e16 & 9.13\% & \cellcolor[rgb]{ .502,  .776,  .486}1.55e11 & 1.66e10 & 49.18\% \\
    $f_{9}$    & \cellcolor[rgb]{ .388,  .745,  .482}1.43e8 & 2.25e7 & \cellcolor[rgb]{ .922,  .898,  .51}\textit{2.36e8} & \textit{2.61e7} & \textit{26.14\%} & \cellcolor[rgb]{ .761,  .851,  .502}2.08e8 & 4.02e7 & 31.15\% & \cellcolor[rgb]{ .973,  .412,  .42}1.01e9 & 1.91e8 & 12.31\% & \cellcolor[rgb]{ .976,  .463,  .431}9.34e8 & 2.05e8 & 8.31\% & \cellcolor[rgb]{ 1,  .914,  .518}2.63e8 & 4.35e7 & 26.75\% \\
    $f_{10}$   & \cellcolor[rgb]{ .4,  .749,  .482}7.70e5 & 6.62e5 & \cellcolor[rgb]{ .388,  .745,  .482}\textit{7.64e5} & \textit{5.90e5} & \textit{85.86\%} & \cellcolor[rgb]{ .835,  .875,  .506}9.33e5 & 5.28e5 & 42.00\% & \cellcolor[rgb]{ .976,  .42,  .424}6.37e7 & 3.54e7 & 3.50\% & \cellcolor[rgb]{ .973,  .412,  .42}6.46e7 & 3.40e7 & 6.87\% & \cellcolor[rgb]{ 1,  .922,  .518}1.06e6 & 4.34e5 & 34.63\% \\
    $f_{11}$   & \cellcolor[rgb]{ .388,  .745,  .482}1.35e7 & 2.21e6 & \cellcolor[rgb]{ .514,  .78,  .486}\textit{3.33e7} & \textit{1.08e7} & \textit{71.54\%} & \cellcolor[rgb]{ 1,  .922,  .518}1.84e8 & 9.53e7 & 22.78\% & \cellcolor[rgb]{ .976,  .459,  .431}1.07e11 & 1.19e11 & 3.15\% & \cellcolor[rgb]{ .973,  .412,  .42}1.17e11 & 1.14e11 & 9.55\% & \cellcolor[rgb]{ .388,  .745,  .482}1.34e7 & 2.60e6 & 83.36\% \\
    $f_{12}$   & \cellcolor[rgb]{ .388,  .745,  .482}2.25e-4 & 2.25e1 & \cellcolor[rgb]{ .918,  .898,  .51}\textit{6.27e2} & \textit{2.11e2} & \textit{32.57\%} & \cellcolor[rgb]{ 1,  .894,  .514}8.12e2 & 2.69e2 & 18.82\% & \cellcolor[rgb]{ .91,  .894,  .51}6.14e2 & 2.65e2 & 31.10\% & \cellcolor[rgb]{ .973,  .412,  .42}2.42e3 & 4.79e2 & 24.56\% & \cellcolor[rgb]{ .976,  .443,  .427}2.33e3 & 1.77e2 & 3.87\% \\
    $f_{13}$   & \cellcolor[rgb]{ .388,  .745,  .482}4.45e5 & 2.17e5 & \cellcolor[rgb]{ .518,  .78,  .486}\textit{1.14e7} & \textit{2.20e6} & \textit{49.79\%} & \cellcolor[rgb]{ 1,  .922,  .518}9.28e7 & 1.05e8 & 22.53\% & \cellcolor[rgb]{ .984,  .604,  .459}3.36e9 & 1.63e9 & 15.64\% & \cellcolor[rgb]{ .973,  .412,  .42}5.34e9 & 1.90e9 & 19.35\% & \cellcolor[rgb]{ .522,  .78,  .486}1.18e7 & 1.61e6 & 45.98\% \\
    $f_{14}$   & \cellcolor[rgb]{ .388,  .745,  .482}5.99e5 & 1.01e5 & \cellcolor[rgb]{ .898,  .89,  .51}\textit{4.35e7} & \textit{6.85e6} & \textit{31.95\%} & \cellcolor[rgb]{ 1,  .922,  .518}5.36e7 & 1.12e7 & 19.59\% & \cellcolor[rgb]{ .98,  .525,  .443}1.25e11 & 1.41e11 & 11.25\% & \cellcolor[rgb]{ .973,  .412,  .42}1.60e11 & 9.10e10 & 15.25\% & \cellcolor[rgb]{ .98,  .914,  .514}5.04e7 & 5.73e6 & 28.65\% \\
    $f_{15}$   & \cellcolor[rgb]{ .388,  .745,  .482}9.71e5 & 4.24e4 & \cellcolor[rgb]{ .6,  .804,  .494}\textit{3.66e6} & \textit{2.32e5} & \textit{66.87\%} & \cellcolor[rgb]{ 1,  .922,  .518}8.78e6 & 2.82e6 & 31.66\% & \cellcolor[rgb]{ 1,  .851,  .506}3.05e8 & 9.52e7 & 16.97\% & \cellcolor[rgb]{ .973,  .412,  .42}2.15e9 & 2.77e9 & 22.65\% & \cellcolor[rgb]{ .992,  .918,  .514}8.59e6 & 1.08e6 & 28.89\% \\
    $F$-rank & \multicolumn{2}{c}{\textcolor[rgb]{ .184,  .459,  .71}{performance}/\textcolor[rgb]{ .329,  .51,  .208}{sim}} & \multicolumn{2}{c}{\textcolor[rgb]{ .184,  .459,  .71}{1.63}} & \textcolor[rgb]{ .329,  .51,  .208}{4.29} & \multicolumn{2}{c}{\textcolor[rgb]{ .184,  .459,  .71}{2.53}} & \textcolor[rgb]{ .329,  .51,  .208}{3.32} & \multicolumn{2}{c}{\textcolor[rgb]{ .184,  .459,  .71}{3.70}} & \textcolor[rgb]{ .329,  .51,  .208}{1.79} & \multicolumn{2}{c}{\textcolor[rgb]{ .184,  .459,  .71}{4.67}} & \textcolor[rgb]{ .329,  .51,  .208}{2.14} & \multicolumn{2}{c}{\textcolor[rgb]{ .184,  .459,  .71}{2.67}} & \textcolor[rgb]{ .329,  .51,  .208}{3.46} \\
    $t$-test & \multicolumn{2}{c}{$</\approx/>$} & \multicolumn{3}{c}{~} & \multicolumn{3}{c}{9/5/1} & \multicolumn{3}{c}{13/2/0} & \multicolumn{3}{c}{14/0/1} & \multicolumn{3}{c}{8/5/2} \\
    \bottomrule
	\multicolumn{18}{m{0.85\textwidth}}{\tiny We cannot compute similarity scores for $f_1$ since the best decision sequence is not unique. For each test case, the greener the indicators, the better the performance.}\\
	\multicolumn{18}{m{0.85\textwidth}}{\tiny For the Friedman tests on the errors, $p = 7.38\times {10}^{-6}$. The smaller the $F$-ranks, the better the performance. For the Friedman tests on the similarity scores, $ p = 2.50 \times {10}^{-3}$. }\\
	\end{tabular}%
  \label{tab:baseline_comparison}%
\end{table*}%

The third perspective focuses on the applicability on practical scenarios. Articulation of heuristics is constrained by how the meta-heuristic will be coupled with them. In existing literature of meta-heuristics, some heuristics are often chosen and fixed and their operations or structure are often exploited, \ie{} the articulated heuristics are not treated as blackboxes, and such coupling makes it difficult for the meta-heuristics to be transferred to other tasks, when new heuristics needed to be articulated. Also, since algorithms are mostly tested on benchmark functions, where the computational resources has only one form, \ie{} the number of function evaluations, seldom investigations have been conducted on the incorporation of different kinds of resources. The flexibility of interface for costs or resources are well-pursued in Bayesian optimization methods \cite{shahriari2016human,cantin2018funneled}. The number of hyperparameters of the meta-heuristics also decides whether they are easy to be applied. The pursuit of simplicity however, is often neglected in the literature of LSBO. The method proposed in this paper is a generalized framework with heavy focus on this perspective and achieves satisfactory compatibility.

\section{Experiments}

\subsection{Experimental Settings}

To validate \algoabbr{}, we test it on two sets of LSBO benchmark functions. The CEC'2013 benchmark suite (LSO13) represents a wide range of real-world LSBO problems. With ill-conditioned complicated sub-components and irregularities \cite{li2013benchmark}, serving to test the overall capabilities of \algoabbr{}. For the test cases in LSO13 suite, the resource for optimization is concretized as the total number of objective function evaluation $\text{maxFEs} = 3\times10^6$ and the problem dimensions are roughly $1000$. The CEC'2008 benchmark problems (LSO08) are naturally scalable to higher dimensions, which serve as the test cases for scalabilityZ. For the test cases in LSO08 , $\text{maxFEs} = 5\times10^{3}D$, where $D$ is the dimensionality of the search space.

We articulate \algoabbr{} with three heuristics, including LS1 (LS) in \cite{latorre2012multiple}, Cooperative Coevolution with random grouping (CC) in \cite{omidvar2010random} and Global Search (GS) in \cite{molina2018iterative}, with the same configurations from the original papers. There are mainly three reasons for which we choose these heuristics: 1) These heuristics are used in literature and have shown powerful empirical performance, \st{} we can refer the knowledge of finetuning these heuristics to their satisfactory configurations; 3) These heuristics have wide intersections with the existing meta-heuristics, which makes the comparison more reasonable. The details of the heuristics are presented in the Appendix.

After a coarse hyperparameter search, we chose $\langle \tau, w \rangle = \langle 1/5, 5 \rangle$. This pair gives the probability of exploration approximately in the interval $(0.01, 0.45)$. We offer more details for in the Appendix.

\begin{table}[!t]
\scriptsize
\setlength{\abovecaptionskip}{-0.05cm}
\setlength{\tabcolsep}{1pt}
\renewcommand\arraystretch{0.8}
  \centering
  \caption{Scalability Tests on LSO08}
    \begin{tabular}{ccccccccc}
    \toprule
    $D$     & \multicolumn{2}{c}{1000} & \multicolumn{2}{c}{2500} & \multicolumn{2}{c}{5000} & \multicolumn{2}{c}{10000} \\
    Problem & mean  & std   & mean  & std   & mean  & std   & mean  & std \\
    \toprule
    $f_1$    & 0.00e0 & 0.00e0 & 0.00e0 & 0.00e0 & 0.00e0 & 0.00e0 & 0.00e0 & 0.00e0 \\
    $f_2$    & 2.60e1 & 2.53e0 & 8.58e1 & 2.12e0 & 1.25e2 & 1.97e0 & 1.44e2 & 7.53e-1 \\
    $f_3$    & 3.26e0 & 3.67e0 & 5.81e2 & 5.99e2 & 1.68e3 & 1.18e3 & 1.61e3 & 1.84e3 \\
    $f_4$    & 0.00e0 & 0.00e0 & 0.00e0 & 0.00e0 & 0.00e0 & 0.00e0 & 0.00e0 & 0.00e0 \\
    $f_5$    & 3.67e-15 & 1.80e-16 & 1.83e-14 & 1.55e-15 & 4.06e-14 & 1.58e-15 & 9.57e-14 & 1.17e-14 \\
    $f_6$    & 1.04e-12 & 5.37e-14 & 5.50e-13 & 2.01e-14 & 1.15e-12 & 3.64e-14 & 2.70e-12 & 1.62e-13 \\
    \bottomrule
	\multicolumn{9}{m{0.45\textwidth}}{\tiny $f_7$ is excluded since we do not know its global minimum.}\\
    \end{tabular}%
  \label{tab:scalability}
\end{table}%

\subsection{Validation}

\begin{table*}[!ht]
\scriptsize
\setlength{\tabcolsep}{1pt}
\renewcommand\arraystretch{0.65}
\centering
\caption{Similarity Scores for LSO08 problems}
\renewcommand\arraystretch{0.65}
    \begin{tabular}{|c|c|c|c|cc|c|c|c|cc|c|c|c|cc|}
    \toprule
    \toprule
    \multicolumn{2}{|c|}{} & \multicolumn{3}{c|}{$f_2$} & \multicolumn{2}{c|}{} & \multicolumn{3}{c|}{$f_4$} & \multicolumn{2}{c|}{} & \multicolumn{3}{c|}{$f_6$} &  \\
    1k  & 1k  & \cellcolor[rgb]{ .388,  .745,  .482}52.14\% & \cellcolor[rgb]{ .71,  .839,  .502}48.34\% & \multicolumn{1}{c|}{\cellcolor[rgb]{ .973,  .412,  .42}23.07\%} & 1k  & 1k  & \cellcolor[rgb]{ .388,  .745,  .482}66.63\% & \cellcolor[rgb]{ .988,  .773,  .486}24.13\% & \multicolumn{1}{c|}{\cellcolor[rgb]{ .973,  .412,  .42}12.75\%} & 1k  & 1k  & \cellcolor[rgb]{ .765,  .855,  .506}70.13\% & \cellcolor[rgb]{ .988,  .765,  .486}66.75\% & \multicolumn{1}{c|}{\cellcolor[rgb]{ .973,  .412,  .42}62.11\%} & 1k \\

    2.5k  & \cellcolor[rgb]{ .388,  .745,  .482}85.40\% & 2.5k  & \cellcolor[rgb]{ .886,  .89,  .514}46.26\% & \multicolumn{1}{c|}{\cellcolor[rgb]{ .98,  .612,  .455}31.68\%} & 2.5k  & \cellcolor[rgb]{ .388,  .745,  .482}55.12\% & 2.5k  & \cellcolor[rgb]{ .776,  .859,  .506}42.67\% & \multicolumn{1}{c|}{\cellcolor[rgb]{ .98,  .627,  .459}19.57\%} & 2.5k  & \cellcolor[rgb]{ .388,  .745,  .482}71.33\% & 2.5k  & \cellcolor[rgb]{ .388,  .745,  .482}72.25\% & \multicolumn{1}{c|}{\cellcolor[rgb]{ .906,  .894,  .514}69.33\%} & 2.5k \\

    5k  & \cellcolor[rgb]{ .957,  .91,  .518}78.46\% & \cellcolor[rgb]{ .722,  .843,  .502}81.33\% & 5k  & \multicolumn{1}{c|}{\cellcolor[rgb]{ .996,  .886,  .51}43.49\%} & 5k  & \cellcolor[rgb]{ .992,  .812,  .494}43.09\% & \cellcolor[rgb]{ .557,  .796,  .494}53.25\% & 5k  & \multicolumn{1}{c|}{\cellcolor[rgb]{ .929,  .902,  .514}33.26\%} & 5k  & \cellcolor[rgb]{ .91,  .898,  .514}51.24\% & \cellcolor[rgb]{ .682,  .831,  .502}60.01\% & 5k  & \multicolumn{1}{c|}{\cellcolor[rgb]{ .996,  .878,  .506}68.24\%} & 5k \\

    10k & \cellcolor[rgb]{ .973,  .412,  .42}71.98\% & \cellcolor[rgb]{ .988,  .71,  .475}75.47\% & \cellcolor[rgb]{ .996,  .871,  .506}77.35\% & \multicolumn{1}{c|}{10k} & 10k & \cellcolor[rgb]{ .973,  .412,  .42}24.25\% & \cellcolor[rgb]{ .98,  .592,  .451}32.75\% & \cellcolor[rgb]{ .459,  .769,  .49}54.33\% & \multicolumn{1}{c|}{10k} & 10k & \cellcolor[rgb]{ .973,  .412,  .42}33.26\% & \cellcolor[rgb]{ .984,  .635,  .463}39.66\% & \cellcolor[rgb]{ .992,  .796,  .49}44.25\% & \multicolumn{1}{c|}{10k} & 10k \\
          & \multicolumn{3}{c|}{$f_1$} & \multicolumn{2}{c|}{} & \multicolumn{3}{c|}{$f_3$} & \multicolumn{2}{c|}{} & \multicolumn{3}{c|}{$f_5$} & \multicolumn{2}{c|}{} \\
    \bottomrule
    \bottomrule
	\multicolumn{16}{m{0.7\textwidth}}{\tiny Each entry is the mean similarity score of decision sequences for a problem explained by the row dimensionality indicator and the column dimensionality indicator. For example, the bottom left $71.98\%$ is the mean value of similarity scores of the $20$ sequences under $D = 1k$ to the $20$ sequences under $D = 10k$. The greener the cell is shaded, the more similar the corresponding decision sequences are. }\\
    \end{tabular}%
  \label{tab:similarity}%
\end{table*}%

\begin{table*}[ht]
\setlength{\tabcolsep}{1.5pt}
\scriptsize
\renewcommand\arraystretch{0.8}
  \centering
  \caption{Comparative Results on LSO13 Problems}
    \begin{tabular}{ccccccccccccccc}
    \toprule
    \multirow{2}[1]{*}{LSO13} & \multicolumn{2}{c}{\textit{\algoabbr{}}} & \multicolumn{2}{c}{MTS} & \multicolumn{2}{c}{MOS} & \multicolumn{2}{c}{CSO} & \multicolumn{2}{c}{CC-CMA-ES} & \multicolumn{2}{c}{DECC-DG2} & \multicolumn{2}{c}{DECC-D} \\
          & \textit{mean} & \textit{std} & mean  & std   & mean  & std   & mean  & std   & mean  & std   & mean  & std   & mean  & std \\
    \toprule
	$f_{1}$    & \cellcolor[rgb]{ .388,  .745,  .482}\textit{\textbf{0.00e0}} & \cellcolor[rgb]{ .859,  .859,  .859}\textit{\textbf{0.00e0}} & \cellcolor[rgb]{ .388,  .745,  .482}\textbf{0.00e0} & \cellcolor[rgb]{ .859,  .859,  .859}\textbf{0.00e0} & \cellcolor[rgb]{ .388,  .745,  .482}\textbf{0.00e0} & \cellcolor[rgb]{ .859,  .859,  .859}\textbf{0.00e0} & \cellcolor[rgb]{ .388,  .745,  .482}\textbf{0.00e0} & \cellcolor[rgb]{ .859,  .859,  .859}\textbf{0.00e0} & \cellcolor[rgb]{ 1,  .922,  .518}3.01e-2 & 1.22e-1 & \cellcolor[rgb]{ .973,  .412,  .42}9.54e5 & 1.57e6 & \cellcolor[rgb]{ 1,  .922,  .518}4.23e-12 & 9.08e-13 \\
    $f_{2}$    & \cellcolor[rgb]{ .388,  .745,  .482}\textit{\textbf{8.69e0}} & \cellcolor[rgb]{ .859,  .859,  .859}\textit{\textbf{2.78e0}} & \cellcolor[rgb]{ 1,  .922,  .518}7.98e2 & 5.41e1 & \cellcolor[rgb]{ .396,  .745,  .482}1.93e1 & 4.16e0 & \cellcolor[rgb]{ .941,  .902,  .514}7.25e2 & 2.72e1 & \cellcolor[rgb]{ 1,  .878,  .51}1.97e3 & 2.73e2 & \cellcolor[rgb]{ .973,  .412,  .42}1.38e4 & 1.78e3 & \cellcolor[rgb]{ 1,  .91,  .518}1.16e3 & 2.92e1 \\
    $f_{3}$    & \cellcolor[rgb]{ 1,  .922,  .518}\textit{9.83e-13} & \textit{5.27e-14} & \cellcolor[rgb]{ .957,  .91,  .514}9.18e-13 & 6.31e-14 & \cellcolor[rgb]{ .388,  .745,  .482}\textbf{0.00e0} & \cellcolor[rgb]{ .859,  .859,  .859}\textbf{0.00e0} & \cellcolor[rgb]{ 1,  .922,  .518}1.19e-12 & 2.11e-14 & \cellcolor[rgb]{ .459,  .765,  .486}1.20e-13 & 3.08e-15 & \cellcolor[rgb]{ .973,  .412,  .42}1.07e1 & 8.71e-1 & \cellcolor[rgb]{ 1,  .922,  .518}8.52e-10 & 5.18e-11 \\
    $f_{4}$    & \cellcolor[rgb]{ .396,  .745,  .482}\textit{6.98e8} & \textit{2.51e8} & \cellcolor[rgb]{ .988,  .671,  .471}2.47e10 & 1.32e10 & \cellcolor[rgb]{ 1,  .882,  .514}1.34e10 & 7.69e9 & \cellcolor[rgb]{ 1,  .922,  .518}1.13e10 & 1.29e9 & \cellcolor[rgb]{ .996,  .918,  .514}1.13e10 & 9.91e9 & \cellcolor[rgb]{ .388,  .745,  .482}\textbf{5.15e8} & \cellcolor[rgb]{ .859,  .859,  .859}\textbf{2.39e8} & \cellcolor[rgb]{ .973,  .412,  .42}3.81e10 & 1.61e10 \\
    $f_{5}$    & \cellcolor[rgb]{ .533,  .784,  .49}\textit{2.68e6} & \textit{4.38e5} & \cellcolor[rgb]{ .984,  .608,  .459}1.02e7 & 1.17e6 & \cellcolor[rgb]{ .973,  .412,  .42}1.11e7 & 1.76e6 & \cellcolor[rgb]{ .388,  .745,  .482}\textbf{7.66e5} & \cellcolor[rgb]{ .859,  .859,  .859}\textbf{1.17e5} & \cellcolor[rgb]{ 1,  .922,  .518}8.72e6 & 2.52e6 & \cellcolor[rgb]{ .522,  .78,  .486}2.51e6 & 4.49e5 & \cellcolor[rgb]{ .992,  .71,  .478}9.71e6 & 1.77e6 \\
    $f_{6}$    & \cellcolor[rgb]{ .604,  .804,  .494}\textit{4.44e4} & \textit{3.42e4} & \cellcolor[rgb]{ .976,  .475,  .431}8.84e5 & 1.66e5 & \cellcolor[rgb]{ .973,  .412,  .42}9.85e5 & 3.22e3 & \cellcolor[rgb]{ .388,  .745,  .482}\textbf{4.36e-8} & \cellcolor[rgb]{ .859,  .859,  .859}\textbf{1.58e-9} & \cellcolor[rgb]{ .98,  .549,  .447}7.55e5 & 3.65e5 & \cellcolor[rgb]{ 1,  .922,  .518}1.25e5 & 2.01e4 & \cellcolor[rgb]{ .463,  .765,  .486}1.57e4 & 2.88e4 \\
    $f_{7}$    & \cellcolor[rgb]{ .388,  .745,  .482}\textit{\textbf{1.58e5}} & \cellcolor[rgb]{ .859,  .859,  .859}\textit{\textbf{3.68e4}} & \cellcolor[rgb]{ 1,  .914,  .518}7.08e7 & 7.60e7 & \cellcolor[rgb]{ 1,  .922,  .518}2.31e7 & 4.12e7 & \cellcolor[rgb]{ .698,  .835,  .498}7.94e6 & 2.88e6 & \cellcolor[rgb]{ .596,  .804,  .494}5.37e6 & 1.15e7 & \cellcolor[rgb]{ 1,  .922,  .518}1.54e7 & 1.01e7 & \cellcolor[rgb]{ .973,  .412,  .42}3.29e9 & 1.14e9 \\
    $f_{8}$    & \cellcolor[rgb]{ .388,  .745,  .482}\textit{\textbf{1.24e11}} & \cellcolor[rgb]{ .859,  .859,  .859}\textit{\textbf{9.29e10}} & \cellcolor[rgb]{ .98,  .494,  .435}1.71e15 & 5.60e14 & \cellcolor[rgb]{ .98,  .522,  .443}1.64e15 & 1.66e15 & \cellcolor[rgb]{ .706,  .835,  .498}3.07e14 & 7.64e13 & \cellcolor[rgb]{ 1,  .922,  .518}5.87e14 & 2.02e14 & \cellcolor[rgb]{ .482,  .773,  .486}9.35e13 & 4.28e13 & \cellcolor[rgb]{ .973,  .412,  .42}1.92e15 & 9.47e14 \\
    $f_{9}$    & \cellcolor[rgb]{ .631,  .816,  .494}\textit{2.36e8} & \textit{2.61e7} & \cellcolor[rgb]{ .988,  .635,  .463}7.32e8 & 9.60e7 & \cellcolor[rgb]{ .973,  .412,  .42}8.97e8 & 1.39e8 & \cellcolor[rgb]{ .388,  .745,  .482}\textbf{4.59e7} & \cellcolor[rgb]{ .859,  .859,  .859}\textbf{8.57e6} & \cellcolor[rgb]{ 1,  .922,  .518}5.19e8 & 1.70e8 & \cellcolor[rgb]{ .722,  .839,  .498}3.06e8 & 7.37e7 & \cellcolor[rgb]{ .988,  .655,  .467}7.18e8 & 1.14e8 \\
    $f_{10}$   & \cellcolor[rgb]{ 1,  .922,  .518}\textit{7.64e5} & \textit{5.90e5} & \cellcolor[rgb]{ 1,  .914,  .518}2.16e6 & 2.45e6 & \cellcolor[rgb]{ .98,  .49,  .435}6.05e7 & 2.91e7 & \cellcolor[rgb]{ .388,  .745,  .482}\textbf{5.35e-5} & \cellcolor[rgb]{ .859,  .859,  .859}\textbf{2.19e-4} & \cellcolor[rgb]{ .973,  .412,  .42}7.11e7 & 2.94e7 & \cellcolor[rgb]{ .388,  .745,  .482}1.43e2 & 1.87e1 & \cellcolor[rgb]{ .388,  .745,  .482}1.03e3 & 1.65e3 \\
    $f_{11}$   & \cellcolor[rgb]{ .388,  .745,  .482}\textit{\textbf{3.33e7}} & \cellcolor[rgb]{ .859,  .859,  .859}\textit{\textbf{1.08e7}} & \cellcolor[rgb]{ 1,  .922,  .518}6.26e9 & 1.64e10 & \cellcolor[rgb]{ .976,  .431,  .424}4.01e10 & 1.23e11 & \cellcolor[rgb]{ .42,  .753,  .482}3.87e8 & 1.13e8 & \cellcolor[rgb]{ .427,  .757,  .482}4.52e8 & 1.18e9 & \cellcolor[rgb]{ 1,  .886,  .514}8.82e9 & 2.60e10 & \cellcolor[rgb]{ .973,  .412,  .42}4.13e10 & 9.40e10 \\
    $f_{12}$   & \cellcolor[rgb]{ .675,  .827,  .498}\textit{6.27e2} & \textit{2.11e2} & \cellcolor[rgb]{ .886,  .886,  .51}1.03e3 & 8.70e2 & \cellcolor[rgb]{ .388,  .745,  .482}\textbf{8.63e1} & \cellcolor[rgb]{ .859,  .859,  .859}\textbf{7.71e1} & \cellcolor[rgb]{ 1,  .922,  .518}1.43e3 & 8.80e1 & \cellcolor[rgb]{ 1,  .922,  .518}1.24e3 & 8.61e1 & \cellcolor[rgb]{ .973,  .412,  .42}1.51e8 & 3.64e8 & \cellcolor[rgb]{ 1,  .922,  .518}1.36e3 & 1.32e2 \\
    $f_{13}$   & \cellcolor[rgb]{ .388,  .745,  .482}\textit{\textbf{1.14e7}} & \cellcolor[rgb]{ .859,  .859,  .859}\textit{\textbf{2.20e6}} & \cellcolor[rgb]{ 1,  .922,  .518}1.30e9 & 1.51e9 & \cellcolor[rgb]{ 1,  .922,  .518}1.13e9 & 7.71e8 & \cellcolor[rgb]{ .69,  .831,  .498}5.65e8 & 1.87e8 & \cellcolor[rgb]{ .996,  .847,  .506}7.42e9 & 4.97e9 & \cellcolor[rgb]{ .906,  .894,  .51}9.62e8 & 3.80e8 & \cellcolor[rgb]{ .973,  .412,  .42}4.33e10 & 9.30e9 \\
    $f_{14}$   & \cellcolor[rgb]{ .388,  .745,  .482}\textit{\textbf{4.35e7}} & \cellcolor[rgb]{ .859,  .859,  .859}\textit{\textbf{6.85e6}} & \cellcolor[rgb]{ 1,  .918,  .518}4.11e10 & 7.79e10 & \cellcolor[rgb]{ .51,  .78,  .486}6.89e9 & 1.41e10 & \cellcolor[rgb]{ 1,  .902,  .514}6.62e10 & 1.30e10 & \cellcolor[rgb]{ .529,  .784,  .49}8.06e9 & 2.05e10 & \cellcolor[rgb]{ 1,  .922,  .518}3.39e10 & 2.15e9 & \cellcolor[rgb]{ .973,  .412,  .42}7.86e11 & 2.91e11 \\
    $f_{15}$   & \cellcolor[rgb]{ .392,  .745,  .482}\textit{3.66e6} & \textit{2.32e5} & \cellcolor[rgb]{ .996,  .816,  .498}4.07e7 & 1.23e7 & \cellcolor[rgb]{ .973,  .412,  .42}1.31e8 & 6.02e7 & \cellcolor[rgb]{ 1,  .922,  .518}1.59e7 & 1.06e6 & \cellcolor[rgb]{ .388,  .745,  .482}\textbf{3.51e6} & \cellcolor[rgb]{ .859,  .859,  .859}\textbf{1.02e6} & \cellcolor[rgb]{ .976,  .914,  .514}1.55e7 & 1.36e6 & \cellcolor[rgb]{ .992,  .757,  .486}5.39e7 & 4.53e6 \\
    $t$-test & \multicolumn{2}{c}{$</\approx/>$} & \multicolumn{2}{c}{\cellcolor[rgb]{ .859,  .859,  .859}\textbf{11/3/1}} & \multicolumn{2}{c}{\cellcolor[rgb]{ .859,  .859,  .859}\textbf{10/3/2}} & \multicolumn{2}{c}{\cellcolor[rgb]{ .859,  .859,  .859}\textbf{10/1/4}} & \multicolumn{2}{c}{\cellcolor[rgb]{ .859,  .859,  .859}\textbf{11/3/1}} & \multicolumn{2}{c}{\cellcolor[rgb]{ .859,  .859,  .859}\textbf{11/2/2}} & \multicolumn{2}{c}{\cellcolor[rgb]{ .859,  .859,  .859}\textbf{12/1/2}} \\
    $F$-rank  & \multicolumn{2}{c}{\cellcolor[rgb]{ .859,  .859,  .859}\textit{\textbf{2.033 }}} & \multicolumn{2}{c}{4.833 } & \multicolumn{2}{c}{4.500 } & \multicolumn{2}{c}{2.933 } & \multicolumn{2}{c}{4.033 } & \multicolumn{2}{c}{4.067 } & \multicolumn{2}{c}{5.600 } \\
	\bottomrule
	\multicolumn{15}{m{0.77\textwidth}}{\tiny The greener the indicator, the better the performance. The best performance is in bold and shaded grey. Friedman test: $\chi^2 = 27.73$, $p = 1.0541\times {10}^{-4}$; If \algoabbr{} performs better in terms of $t$-test results, the corresponding ``l/u/g'' string is in \textbf{bold} type and shaded grey.}
    \end{tabular}%
  \label{tab:LSO13}%
\end{table*}%

\subsubsection{Comprehensive Abilities}
We validate the effectiveness of \algoabbr{} by comparing it with several baselines that articulate the same heuristics. We run \algoabbr{} and the baselines on the $15$ test cases in LSO13 for $20$ times and the obtained results are listed in Table \ref{tab:baseline_comparison}. The operation of ``RAN'' means purely random choices for actions. The ``BEK'' means the best known average performance achieved by using from LS, CC and GS. The best known decision sequences come from many random runs using the three heuristics\footnote{The BEK sequences will be released together with the source code.}. The performance obtained using the ``BEK'' baseline indicates the full potential of the articulated framework and serves as the approximate lower bound of the error obtained by \algoabbr{}. We also present the decision sequence similarity scores of \algoabbr{} with respect to the sequences used in the BEK baseline. The sequence similarity scores are computed as mean values of the local alignment scores via the Smith-Waterman method \cite{smith1981identification}. The higher the similarity scores, the more similar the decision sequences are to the best known sequence. Some additional curves are provided in the Appendix.

To show the general performance rankings of algorithms, Friedman tests are conducted with significance level $\alpha = 0.05$. If Friedman tests tell $p < \alpha$, we can conclude that the given performance rankings are statistically meaningful. To show one-on-one performance differences of algorithms, paired $t$-tests are also conducted with significance level of $\alpha = 0.05$. We give collective $t$-test results in the format of a ``$</\approx/>$'' string. $<$ represents the number of test cases in which we are highly-confident that the first algorithm gives smaller results than the second one. $>$ represents the opposite of $<$. $\approx$ represents the number of test results with no confident difference. We obtain the following observations:

\begin{enumerate}
\item
According to the statistics, \algoabbr{} achieves significantly better performance than the $4$ baselines (RAN, LS, CC and GS);  This in a way shows the effectiveness of the online decisions.
\item
In terms of decision similarity, \algoabbr{} is significantly more similar to BEK, than the other baselines. This indicates that \algoabbr{} may be more able to excavate effective sequences of actions.
\end{enumerate}

\begin{table*}[htbp]
\setlength{\tabcolsep}{1.5pt}
\scriptsize
\renewcommand\arraystretch{0.8}
  \centering
  \caption{Comparative Results on LSO08 Problems}
    \begin{tabular}{cccccccccccccc}
    \toprule
    \multirow{2}[1]{*}{LSO08} &       & \multicolumn{2}{c}{\textit{\algoabbr}} & \multicolumn{2}{c}{MTS} & \multicolumn{2}{c}{CSO} & \multicolumn{2}{c}{CC-CMA-ES} & \multicolumn{2}{c}{DECC-DG2} & \multicolumn{2}{c}{DECC-D} \\
          & D     & \textit{mean} & \textit{std} & mean  & std   & mean  & std   & mean  & std   & mean  & std   & mean  & std \\
    \toprule
	\multirow{4}[0]{*}{$f_1$} & 1000  & \cellcolor[rgb]{ .388,  .745,  .482}\textit{\textbf{0.00e0}} & \cellcolor[rgb]{ .859,  .859,  .859}\textit{\textbf{0.00e0}} & \cellcolor[rgb]{ .388,  .745,  .482}\textbf{0.00e0} & \cellcolor[rgb]{ .859,  .859,  .859}\textbf{0.00e0} & \cellcolor[rgb]{ .388,  .745,  .482}\textbf{0.00e0} & \cellcolor[rgb]{ .859,  .859,  .859}\textbf{0.00e0} & \cellcolor[rgb]{ .388,  .745,  .482}\textbf{0.00e0} & \cellcolor[rgb]{ .859,  .859,  .859}\textbf{0.00e0} & \cellcolor[rgb]{ .973,  .412,  .42}1.08e2 & 4.40e2 & \cellcolor[rgb]{ .388,  .745,  .482}\textbf{0.00e0} & \cellcolor[rgb]{ .859,  .859,  .859}\textbf{0.00e0} \\
          & 2500  & \cellcolor[rgb]{ .388,  .745,  .482}\textit{\textbf{0.00e0}} & \cellcolor[rgb]{ .859,  .859,  .859}\textit{\textbf{0.00e0}} & \cellcolor[rgb]{ .388,  .745,  .482}\textbf{0.00e0} & \cellcolor[rgb]{ .859,  .859,  .859}\textbf{0.00e0} & \cellcolor[rgb]{ 1,  .922,  .518}1.48e-20 & 7.71e-22 & \cellcolor[rgb]{ .388,  .745,  .482}\textbf{0.00e0} & \cellcolor[rgb]{ .859,  .859,  .859}\textbf{0.00e0} & \cellcolor[rgb]{ .973,  .412,  .42}1.11e5 & 1.25e5 & \cellcolor[rgb]{ .388,  .745,  .482}\textbf{0.00e0} & \cellcolor[rgb]{ .859,  .859,  .859}\textbf{0.00e0} \\
          & 5000  & \cellcolor[rgb]{ .388,  .745,  .482}\textit{\textbf{0.00e0}} & \cellcolor[rgb]{ .859,  .859,  .859}\textit{\textbf{0.00e0}} & \cellcolor[rgb]{ .388,  .745,  .482}\textbf{0.00e0} & \cellcolor[rgb]{ .859,  .859,  .859}\textbf{0.00e0} & \cellcolor[rgb]{ 1,  .922,  .518}5.32e-18 & 4.63e-19 & \cellcolor[rgb]{ .388,  .745,  .482}\textbf{0.00e0} & \cellcolor[rgb]{ .859,  .859,  .859}\textbf{0.00e0} & \cellcolor[rgb]{ .973,  .412,  .42}6.46e5 & 6.31e4 & \cellcolor[rgb]{ .388,  .745,  .482}\textbf{0.00e0} & \cellcolor[rgb]{ .859,  .859,  .859}\textbf{0.00e0} \\
          & 10000 & \cellcolor[rgb]{ .388,  .745,  .482}\textit{\textbf{0.00e0}} & \cellcolor[rgb]{ .859,  .859,  .859}\textit{\textbf{0.00e0}} & \cellcolor[rgb]{ .388,  .745,  .482}\textbf{0.00e0} & \cellcolor[rgb]{ .859,  .859,  .859}\textbf{0.00e0} & \cellcolor[rgb]{ 1,  .922,  .518}1.64e-17 & 2.02e-18 & \cellcolor[rgb]{ .388,  .745,  .482}\textbf{0.00e0} & \cellcolor[rgb]{ .859,  .859,  .859}\textbf{0.00e0} & \cellcolor[rgb]{ .973,  .412,  .42}6.47e7 & 6.94e5 & \cellcolor[rgb]{ .388,  .745,  .482}\textbf{0.00e0} & \cellcolor[rgb]{ .859,  .859,  .859}\textbf{0.00e0} \\
    \multirow{4}[0]{*}{$f_2$} & 1000  & \cellcolor[rgb]{ .388,  .745,  .482}\textit{\textbf{2.60e1}} & \cellcolor[rgb]{ .859,  .859,  .859}\textit{\textbf{2.53e0}} & \cellcolor[rgb]{ .992,  .714,  .478}1.12e2 & 8.86e0 & \cellcolor[rgb]{ 1,  .906,  .518}8.04e1 & 2.43e0 & \cellcolor[rgb]{ .973,  .412,  .42}1.62e2 & 7.74e0 & \cellcolor[rgb]{ .961,  .91,  .514}7.40e1 & 1.66e0 & \cellcolor[rgb]{ .749,  .847,  .502}5.65e1 & 4.63e0 \\
          & 2500  & \cellcolor[rgb]{ .753,  .847,  .502}\textit{8.58e1} & \textit{2.12e0} & \cellcolor[rgb]{ .988,  .671,  .471}1.47e2 & 1.30e0 & \cellcolor[rgb]{ .388,  .745,  .482}\textbf{4.61e1} & \cellcolor[rgb]{ .859,  .859,  .859}\textbf{1.28e0} & \cellcolor[rgb]{ .973,  .412,  .42}1.82e2 & 1.14e1 & \cellcolor[rgb]{ .992,  .725,  .482}1.39e2 & 2.96e0 & \cellcolor[rgb]{ .58,  .8,  .49}6.72e1 & 4.56e0 \\
          & 5000  & \cellcolor[rgb]{ .882,  .886,  .51}\textit{1.25e2} & \textit{1.97e0} & \cellcolor[rgb]{ .988,  .698,  .475}1.59e2 & 1.12e0 & \cellcolor[rgb]{ .392,  .745,  .482}8.26e1 & 1.30e0 & \cellcolor[rgb]{ .973,  .412,  .42}1.90e2 & 5.81e0 & \cellcolor[rgb]{ .996,  .831,  .502}1.44e2 & 3.78e0 & \cellcolor[rgb]{ .388,  .745,  .482}\textbf{8.21e1} & \cellcolor[rgb]{ .859,  .859,  .859}\textbf{4.23e0} \\
          & 10000 & \cellcolor[rgb]{ .89,  .886,  .51}\textit{1.44e2} & \textit{7.53e-1} & \cellcolor[rgb]{ .992,  .761,  .49}1.69e2 & 1.32e0 & \cellcolor[rgb]{ .702,  .835,  .498}1.23e2 & 1.48e0 & \cellcolor[rgb]{ .976,  .427,  .424}1.95e2 & 5.27e-1 & \cellcolor[rgb]{ .973,  .412,  .42}1.96e2 & 8.58e-1 & \cellcolor[rgb]{ .388,  .745,  .482}\textbf{8.69e1} & \cellcolor[rgb]{ .859,  .859,  .859}\textbf{4.21e0} \\
    \multirow{4}[0]{*}{$f_3$} & 1000  & \cellcolor[rgb]{ .388,  .745,  .482}\textit{\textbf{3.26e0}} & \cellcolor[rgb]{ .859,  .859,  .859}\textit{\textbf{3.67e0}} & \cellcolor[rgb]{ .478,  .769,  .486}1.69e2 & 1.27e2 & \cellcolor[rgb]{ 1,  .922,  .518}1.26e3 & 1.42e2 & \cellcolor[rgb]{ .941,  .902,  .514}1.02e3 & 2.87e1 & \cellcolor[rgb]{ .973,  .412,  .42}5.49e6 & 1.71e7 & \cellcolor[rgb]{ 1,  .922,  .518}1.23e3 & 1.10e2 \\
          & 2500  & \cellcolor[rgb]{ .388,  .745,  .482}\textit{\textbf{5.81e2}} & \cellcolor[rgb]{ .859,  .859,  .859}\textit{\textbf{5.99e2}} & \cellcolor[rgb]{ .439,  .761,  .482}7.68e2 & 2.50e2 & \cellcolor[rgb]{ .957,  .906,  .514}2.54e3 & 2.48e1 & \cellcolor[rgb]{ 1,  .922,  .518}2.82e3 & 8.96e1 & \cellcolor[rgb]{ .973,  .412,  .42}7.37e9 & 4.93e9 & \cellcolor[rgb]{ 1,  .922,  .518}3.16e3 & 4.91e2 \\
          & 5000  & \cellcolor[rgb]{ .435,  .757,  .482}\textit{1.68e3} & \textit{1.18e3} & \cellcolor[rgb]{ .388,  .745,  .482}\textbf{1.35e3} & \cellcolor[rgb]{ .859,  .859,  .859}\textbf{3.55e2} & \cellcolor[rgb]{ .976,  .914,  .514}5.45e3 & 1.36e2 & \cellcolor[rgb]{ 1,  .922,  .518}5.75e3 & 1.89e2 & \cellcolor[rgb]{ .973,  .412,  .42}2.22e11 & 6.49e10 & \cellcolor[rgb]{ 1,  .922,  .518}6.19e3 & 4.41e2 \\
          & 10000 & \cellcolor[rgb]{ .388,  .745,  .482}\textit{\textbf{1.61e3}} & \cellcolor[rgb]{ .859,  .859,  .859}\textit{\textbf{1.84e3}} & \cellcolor[rgb]{ .412,  .749,  .482}2.04e3 & 6.95e2 & \cellcolor[rgb]{ 1,  .922,  .518}1.57e4 & 8.75e2 & \cellcolor[rgb]{ .98,  .914,  .514}1.16e4 & 2.97e2 & \cellcolor[rgb]{ .973,  .412,  .42}8.39e13 & 1.21e12 & \cellcolor[rgb]{ 1,  .922,  .518}1.22e4 & 4.06e2 \\
    \multirow{4}[0]{*}{$f_4$} & 1000  & \cellcolor[rgb]{ .388,  .745,  .482}\textit{\textbf{0.00e0}} & \cellcolor[rgb]{ .859,  .859,  .859}\textit{\textbf{0.00e0}} & \cellcolor[rgb]{ .388,  .745,  .482}\textbf{0.00e0} & \cellcolor[rgb]{ .859,  .859,  .859}\textbf{0.00e0} & \cellcolor[rgb]{ 1,  .914,  .518}7.05e2 & 3.00e1 & \cellcolor[rgb]{ .992,  .757,  .486}1.93e3 & 1.29e2 & \cellcolor[rgb]{ .973,  .412,  .42}4.63e3 & 4.91e2 & \cellcolor[rgb]{ .906,  .894,  .51}5.21e2 & 2.27e1 \\
          & 2500  & \cellcolor[rgb]{ .388,  .745,  .482}\textit{\textbf{0.00e0}} & \cellcolor[rgb]{ .859,  .859,  .859}\textit{\textbf{0.00e0}} & \cellcolor[rgb]{ .388,  .745,  .482}7.96e-1 & 4.45e-1 & \cellcolor[rgb]{ .996,  .918,  .514}1.22e3 & 4.07e1 & \cellcolor[rgb]{ .996,  .788,  .494}5.31e3 & 2.48e2 & \cellcolor[rgb]{ .973,  .412,  .42}1.68e4 & 1.63e3 & \cellcolor[rgb]{ 1,  .922,  .518}1.23e3 & 6.12e1 \\
          & 5000  & \cellcolor[rgb]{ .388,  .745,  .482}\textit{\textbf{0.00e0}} & \cellcolor[rgb]{ .859,  .859,  .859}\textit{\textbf{0.00e0}} & \cellcolor[rgb]{ .388,  .745,  .482}3.98e0 & 1.22e0 & \cellcolor[rgb]{ 1,  .922,  .518}2.84e3 & 3.29e1 & \cellcolor[rgb]{ .996,  .812,  .498}1.10e4 & 1.39e2 & \cellcolor[rgb]{ .973,  .412,  .42}4.09e4 & 3.73e3 & \cellcolor[rgb]{ .945,  .902,  .514}2.38e3 & 6.84e1 \\
          & 10000 & \cellcolor[rgb]{ .388,  .745,  .482}\textit{\textbf{0.00e0}} & \cellcolor[rgb]{ .859,  .859,  .859}\textit{\textbf{0.00e0}} & \cellcolor[rgb]{ .388,  .745,  .482}8.29e0 & 3.91e0 & \cellcolor[rgb]{ 1,  .918,  .518}9.19e3 & 1.63e2 & \cellcolor[rgb]{ 1,  .894,  .514}2.22e4 & 5.49e2 & \cellcolor[rgb]{ .973,  .412,  .42}2.62e5 & 3.92e2 & \cellcolor[rgb]{ .8,  .863,  .506}4.68e3 & 5.64e1 \\
    \multirow{4}[0]{*}{$f_5$} & 1000  & \cellcolor[rgb]{ .996,  .918,  .514}\textit{3.67e-15} & \textit{1.80e-16} & \cellcolor[rgb]{ 1,  .922,  .518}3.71e-15 & 1.37e-16 & \cellcolor[rgb]{ .388,  .745,  .482}\textbf{2.22e-16} & \cellcolor[rgb]{ .859,  .859,  .859}\textbf{0.00e0} & \cellcolor[rgb]{ 1,  .922,  .518}2.71e-3 & 5.92e-3 & \cellcolor[rgb]{ .973,  .412,  .42}4.46e-1 & 5.78e-1 & \cellcolor[rgb]{ .651,  .82,  .494}1.72e-15 & 9.17e-17 \\
          & 2500  & \cellcolor[rgb]{ .388,  .745,  .482}\textit{1.83e-14} & \textit{1.55e-15} & \cellcolor[rgb]{ .388,  .745,  .482}1.92e-14 & 4.78e-16 & \cellcolor[rgb]{ .388,  .745,  .482}\textbf{4.44e-16} & \cellcolor[rgb]{ .859,  .859,  .859}\textbf{0.00e0} & \cellcolor[rgb]{ 1,  .922,  .518}1.97e-3 & 4.41e-3 & \cellcolor[rgb]{ .973,  .412,  .42}1.35e3 & 2.05e2 & \cellcolor[rgb]{ 1,  .922,  .518}3.45e-3 & 7.70e-3 \\
          & 5000  & \cellcolor[rgb]{ 1,  .922,  .518}\textit{4.06e-14} & \textit{1.58e-15} & \cellcolor[rgb]{ 1,  .922,  .518}6.53e-14 & 3.24e-16 & \cellcolor[rgb]{ .388,  .745,  .482}\textbf{6.66e-16} & \cellcolor[rgb]{ .859,  .859,  .859}\textbf{0.00e0} & \cellcolor[rgb]{ .8,  .863,  .506}2.11e-14 & 6.97e-15 & \cellcolor[rgb]{ .973,  .412,  .42}1.65e4 & 1.23e3 & \cellcolor[rgb]{ .608,  .808,  .494}1.17e-14 & 1.45e-16 \\
          & 10000 & \cellcolor[rgb]{ 1,  .922,  .518}\textit{9.57e-14} & \textit{1.17e-14} & \cellcolor[rgb]{ 1,  .922,  .518}7.07e-5 & 1.58e-4 & \cellcolor[rgb]{ .388,  .745,  .482}\textbf{1.13e-15} & \cellcolor[rgb]{ .859,  .859,  .859}\textbf{4.97e-17} & \cellcolor[rgb]{ .745,  .847,  .502}4.02e-14 & 2.48e-14 & \cellcolor[rgb]{ .973,  .412,  .42}5.85e5 & 5.46e3 & \cellcolor[rgb]{ .592,  .804,  .494}2.35e-14 & 1.57e-16 \\
    \multirow{4}[0]{*}{$f_6$} & 1000  & \cellcolor[rgb]{ 1,  .922,  .518}\textit{1.04e-12} & \textit{5.37e-14} & \cellcolor[rgb]{ .957,  .906,  .514}9.22e-13 & 4.28e-14 & \cellcolor[rgb]{ 1,  .922,  .518}1.20e-12 & 1.61e-14 & \cellcolor[rgb]{ .396,  .745,  .482}1.17e-13 & 3.25e-15 & \cellcolor[rgb]{ .973,  .412,  .42}1.09e1 & 7.98e-1 & \cellcolor[rgb]{ .388,  .745,  .482}\textbf{1.05e-13} & \cellcolor[rgb]{ .859,  .859,  .859}\textbf{2.65e-15} \\
          & 2500  & \cellcolor[rgb]{ .463,  .765,  .486}\textit{5.50e-13} & \textit{2.01e-14} & \cellcolor[rgb]{ .918,  .898,  .51}2.31e-12 & 4.63e-13 & \cellcolor[rgb]{ 1,  .922,  .518}2.92e-12 & 4.44e-14 & \cellcolor[rgb]{ .996,  .796,  .494}3.61e0 & 8.07e0 & \cellcolor[rgb]{ .973,  .412,  .42}1.46e1 & 2.93e-1 & \cellcolor[rgb]{ .388,  .745,  .482}\textbf{2.62e-13} & \cellcolor[rgb]{ .859,  .859,  .859}\textbf{6.16e-15} \\
          & 5000  & \cellcolor[rgb]{ .404,  .749,  .482}\textit{1.15e-12} & \textit{3.64e-14} & \cellcolor[rgb]{ .463,  .765,  .486}3.53e-12 & 8.94e-13 & \cellcolor[rgb]{ 1,  .922,  .518}4.51e-11 & 7.07e-13 & \cellcolor[rgb]{ .973,  .412,  .42}1.82e1 & 8.90e-2 & \cellcolor[rgb]{ .976,  .431,  .424}1.75e1 & 9.26e-1 & \cellcolor[rgb]{ .388,  .745,  .482}\textbf{5.13e-13} & \cellcolor[rgb]{ .859,  .859,  .859}\textbf{5.55e-15} \\
          & 10000 & \cellcolor[rgb]{ .4,  .745,  .482}\textit{2.70e-12} & \textit{1.62e-13} & \cellcolor[rgb]{ .42,  .753,  .482}5.65e-12 & 2.53e-13 & \cellcolor[rgb]{ 1,  .922,  .518}1.66e-10 & 2.62e-11 & \cellcolor[rgb]{ .976,  .486,  .435}1.85e1 & 7.29e-1 & \cellcolor[rgb]{ .973,  .412,  .42}2.16e1 & 6.60e-3 & \cellcolor[rgb]{ .388,  .745,  .482}\textbf{1.03e-12} & \cellcolor[rgb]{ .859,  .859,  .859}\textbf{1.29e-14} \\
    \multirow{4}[0]{*}{$t$-test} & 1000  & \multicolumn{2}{c}{\textit{$\sim$}} & \multicolumn{2}{c}{\cellcolor[rgb]{ .859,  .859,  .859}\textbf{2/3/1}} & \multicolumn{2}{c}{\cellcolor[rgb]{ .859,  .859,  .859}\textbf{4/1/1}} & \multicolumn{2}{c}{\cellcolor[rgb]{ .859,  .859,  .859}\textbf{3/2/1}} & \multicolumn{2}{c}{\cellcolor[rgb]{ .859,  .859,  .859}\textbf{5/1/0}} & \multicolumn{2}{c}{\cellcolor[rgb]{ .859,  .859,  .859}\textbf{3/1/2}} \\
          & 2500  & \multicolumn{2}{c}{\textit{$\sim$}} & \multicolumn{2}{c}{\cellcolor[rgb]{ .859,  .859,  .859}\textbf{5/1/0}} & \multicolumn{2}{c}{\cellcolor[rgb]{ .859,  .859,  .859}\textbf{4/0/2}} & \multicolumn{2}{c}{\cellcolor[rgb]{ .859,  .859,  .859}\textbf{5/1/0}} & \multicolumn{2}{c}{\cellcolor[rgb]{ .859,  .859,  .859}\textbf{6/0/0}} & \multicolumn{2}{c}{\cellcolor[rgb]{ .859,  .859,  .859}\textbf{3/1/2}} \\
          & 5000  & \multicolumn{2}{c}{\textit{$\sim$}} & \multicolumn{2}{c}{\cellcolor[rgb]{ .859,  .859,  .859}\textbf{4/2/0}} & \multicolumn{2}{c}{\cellcolor[rgb]{ .859,  .859,  .859}\textbf{4/0/2}} & \multicolumn{2}{c}{\cellcolor[rgb]{ .859,  .859,  .859}\textbf{4/1/1}} & \multicolumn{2}{c}{\cellcolor[rgb]{ .859,  .859,  .859}\textbf{6/0/0}} & \multicolumn{2}{c}{2/1/3} \\
          & 10000 & \multicolumn{2}{c}{\textit{$\sim$}} & \multicolumn{2}{c}{\cellcolor[rgb]{ .859,  .859,  .859}\textbf{3/3/0}} & \multicolumn{2}{c}{\cellcolor[rgb]{ .859,  .859,  .859}\textbf{4/0/2}} & \multicolumn{2}{c}{\cellcolor[rgb]{ .859,  .859,  .859}\textbf{4/1/1}} & \multicolumn{2}{c}{\cellcolor[rgb]{ .859,  .859,  .859}\textbf{6/0/0}} & \multicolumn{2}{c}{2/1/3} \\
    \multirow{4}[1]{*}{$F$-rank} & 1000  & \multicolumn{2}{c}{\cellcolor[rgb]{ .859,  .859,  .859}\textit{\textbf{2.25}}} & \multicolumn{2}{c}{3.08} & \multicolumn{2}{c}{3.67} & \multicolumn{2}{c}{4.00} & \multicolumn{2}{c}{5.50} & \multicolumn{2}{c}{2.50} \\
          & 2500  & \multicolumn{2}{c}{\cellcolor[rgb]{ .859,  .859,  .859}\textit{\textbf{1.92}}} & \multicolumn{2}{c}{2.92} & \multicolumn{2}{c}{2.83} & \multicolumn{2}{c}{4.42} & \multicolumn{2}{c}{5.67} & \multicolumn{2}{c}{3.25} \\
          & 5000  & \multicolumn{2}{c}{\cellcolor[rgb]{ .859,  .859,  .859}\textit{\textbf{2.42}}} & \multicolumn{2}{c}{3.08} & \multicolumn{2}{c}{3.17} & \multicolumn{2}{c}{4.42} & \multicolumn{2}{c}{5.50} & \multicolumn{2}{c}{\cellcolor[rgb]{ .859,  .859,  .859}\textbf{2.42}} \\
          & 10000 & \multicolumn{2}{c}{\cellcolor[rgb]{ .859,  .859,  .859}\textit{\textbf{2.25}}} & \multicolumn{2}{c}{3.08} & \multicolumn{2}{c}{3.50} & \multicolumn{2}{c}{3.92} & \multicolumn{2}{c}{6.00} & \multicolumn{2}{c}{\cellcolor[rgb]{ .859,  .859,  .859}\textbf{2.25}} \\
    \bottomrule
	\multicolumn{14}{m{0.7\textwidth}}{\tiny MOS is excluded because we have trouble reproducing the algorithm and cannot find the related results; All Friedman tests satisfy $p \ll 5\%$. The best ranking is in \textbf{bold} type and shaded grey; If \algoabbr{} performs better in terms of $t$-test results, the corresponding ``l/u/g'' string is in \textbf{bold} type and shaded grey.}
    \end{tabular}%
  \label{tab:LSO08}%
\end{table*}%

\subsection{Scalability}
We test if \algoabbr{} has stable behavior to achieve satisfactory performance within wide range of problem dimensionality by validating it on the $6$ scalable benchmark functions in LSO08 with $D \in \{1000, 2500, 5000, 10000\}$. $\text{maxFEs}$ scales linearly with problem dimensionality, which means that the length of the decision sequence on different dimensions are roughly the same and thus are comparable. In Table \ref{tab:scalability}, the means and stds obtained over $20$ independent runs on each case are given. In Table \ref{tab:similarity}, the similarity score matrices of the decision sequences are presented.

From the similarity matrices, it can be observed that generally the larger difference in problem dimensionality, the less similar the sequences are. However, these dissimilarities still leads to similar results: on $4$ of $6$ problems, the global optima are found. This indicates that, though the patterns of sequences are not similar when the problem dimensionality changes, the qualities of the decisions remain strong.

\subsection{Comparisons with Existing Algorithms}
\subsubsection{Comprehensive Abilities}
On the LSO13 benchmark problems, we compare the performance of \algoabbr{} with some competitive algorithms, including MOS \cite{latorre2012multiple}, MTS \cite{tseng2008multiple}, CSO \cite{cheng2015competitive}, CC-CMA-ES \cite{liu2013scaling}, DECC-DG2 \cite{omidvar2017faster} and DECC-D \cite{omidvar2010cooperative}. Under the CEC'2018 competition standards, we run each algorithm independently for $20$ times on each benchmark function. We present the results in Table \ref{tab:LSO13}.

The statistics show that \algoabbr{} achieves the best results within the compared algorithms. Furthermore, on $6$ test cases, \algoabbr{} achieved errors at least one order of magnitude lower than all the others. Since we do not add any additional prior-dependent components such as restart mechanisms to further enhance the performance, such performance can be highlighted.

\subsection{Scalability}
For this part, we adopt the same problem settings as in the scalability validation. We compare the performance of \algoabbr{} against the other algorithms, whose results are listed in Table \ref{tab:LSO08}. The statistics show \algoabbr{} achieves the best performance, exhibiting less performance deterioration than other algorithms with the increment of problem dimensionality. Some additional curves are provided in the Appendix.

\section{Conclusion}

Purposely for addressing the existing problems of meta-heuristic frameworks, this paper formulates a methodology of transforming zero-order optimization problems into decision processes, in which during different stages for decision, different articulated heuristics are initiated as actions.

With such methodology, a solution using local reward estimation is proposed, with the hypothesis that the problem can be approximately addressed from the perspective of multi-armed bandits with non-stationary reward distributions. The temporal estimation is implemented using a simple window. With the local estimations, the problem is then addressed using Boltzmann exploration. This proposed solution enables robust interfaces for practical use and simplicity to ensure easy generalization, accompanied with bounds for the behavior and the guidelines for hyperparameter tuning. Empirically, the proposed \algoabbr{}, when articulated with three heuristics, has shown significant performance when compared to baselines on many benchmark problems, without embedding prior-dependent components. When compared to other competitive existing methods, it shows nearly state-of-the-art performance.

\clearpage
\bibliographystyle{aaai}
\bibliography{refs}

\begin{thebibliography}{}

\bibitem[\protect\citeauthoryear{Al-Dujaili and
  Sundaram}{2017}]{Al-Dujail2017EMB}
Al-Dujaili, A., and Sundaram, S.
\newblock 2017.
\newblock Embedded bandits for large-scale black-box optimization.
\newblock In {\em AAAI}.

\bibitem[\protect\citeauthoryear{Boluf-Rohler, Fiol-Gonzalez, and
  Chen}{2015}]{rohler2015minimum}
Boluf-Rohler, A.; Fiol-Gonzalez, S.; and Chen, S.
\newblock 2015.
\newblock A minimum population search hybrid for large scale global
  optimization.
\newblock In {\em IEEE Congr. Evol. Comput.},  1958--1965.

\bibitem[\protect\citeauthoryear{Cao \bgroup et al\mbox.\egroup
  }{2019}]{cao2018multimodal}
Cao, Y.; Zhang, H.; Li, W.; Zhou, M.; Zhang, Y.; and Chaovalitwongse, W.~A.
\newblock 2019.
\newblock Comprehensive learning particle swarm optimization algorithm with
  local search for multimodal functions.
\newblock {\em IEEE Trans. Evol. Comput.} 23(4):718--731.

\bibitem[\protect\citeauthoryear{Cheng and Jin}{2015}]{cheng2015competitive}
Cheng, R., and Jin, Y.
\newblock 2015.
\newblock A competitive swarm optimizer for large scale optimization.
\newblock {\em IEEE Trans. Cybern.} 45(2):191--204.

\bibitem[\protect\citeauthoryear{Ge \bgroup et al\mbox.\egroup
  }{2017}]{ge2017cooperative}
Ge, H.; Sun, L.; Tan, G.; Chen, Z.; and Chen, C. L.~P.
\newblock 2017.
\newblock Cooperative hierarchical pso with two stage variable interaction
  reconstruction for large scale optimization.
\newblock {\em IEEE Trans. Cybern.} 47(9):2809--2823.

\bibitem[\protect\citeauthoryear{Ge \bgroup et al\mbox.\egroup
  }{2018}]{zhan2018dis}
Ge, Y.; Yu, W.; Lin, Y.; Gong, Y.; Zhan, Z.; Chen, W.; and Zhang, J.
\newblock 2018.
\newblock Distributed differential evolution based on adaptive mergence and
  split for large-scale optimization.
\newblock {\em IEEE Trans. Cybern.} 48(7):2166--2180.

\bibitem[\protect\citeauthoryear{Hansen and Auger}{2014}]{hansen2014principled}
Hansen, N., and Auger, A.
\newblock 2014.
\newblock Principled design of continuous stochastic search: From theory to
  practice.
\newblock {\em Theory and Principled Methods for the Design of Metaheuristics}
  145--180.

\bibitem[\protect\citeauthoryear{LaTorre, Muelas, and
  Pena}{2012}]{latorre2012multiple}
LaTorre, A.; Muelas, S.; and Pena, J.~M.
\newblock 2012.
\newblock Multiple offspring sampling in large scale global optimization.
\newblock In {\em IEEE Congr. Evol. Comput.}

\bibitem[\protect\citeauthoryear{Li \bgroup et al\mbox.\egroup
  }{2013}]{li2013benchmark}
Li, X.; Tang, K.; Omidvar, M.~N.; Yang, Z.; and Qin, K.
\newblock 2013.
\newblock Benchmark functions for the {CEC'2013} special session and
  competition on large-scale global optimization.
\newblock Technical report, Evolutionary Computing and Machine Learning group,
  RMIT, Australia.

\bibitem[\protect\citeauthoryear{Li \bgroup et al\mbox.\egroup
  }{2018}]{li2018fast}
Li, Z.; Zhang, Q.; Lin, X.; and Zhen, H.
\newblock 2018.
\newblock Fast covariance matrix adaptation for large-scale black-box
  optimization.
\newblock {\em IEEE Trans. Cybern.}

\bibitem[\protect\citeauthoryear{Liu and Tang}{2013}]{liu2013scaling}
Liu, J., and Tang, K.
\newblock 2013.
\newblock Scaling up covariance matrix adaptation evolution strategy using
  cooperative coevolution.
\newblock In {\em International Conference on Intelligent Data Engineering and
  Automated Learning},  350--357.

\bibitem[\protect\citeauthoryear{Loshchilov, Glasmachers, and
  Beyer}{2019}]{loshchilov2018limited}
Loshchilov, I.; Glasmachers, T.; and Beyer, H.
\newblock 2019.
\newblock Large scale black-box optimization by limited-memory matrix
  adaptation.
\newblock {\em IEEE Trans. Evol. Comput.} 23(2):353--358.

\bibitem[\protect\citeauthoryear{Lu \bgroup et al\mbox.\egroup
  }{2018}]{cheng2018incremental}
Lu, X.; Menzel, S.; Tang, K.; and Yao, X.
\newblock 2018.
\newblock Cooperative co-evolution based design optimisation: A concurrent
  engineering perspective.
\newblock {\em IEEE Trans. Evol. Comput.} 22(2):173--188.

\bibitem[\protect\citeauthoryear{Ma \bgroup et al\mbox.\egroup
  }{2019}]{ma2018survey}
Ma, X.; Li, X.; Zhang, Q.; Tang, K.; Liang, Z.; Xie, W.; and Zhu, Z.
\newblock 2019.
\newblock A survey on cooperative co-evolutionary algorithms.
\newblock {\em IEEE Trans. Evol. Comput.} 23(3):421--441.

\bibitem[\protect\citeauthoryear{Martinez-Cantin}{2019}]{cantin2018funneled}
Martinez-Cantin, R.
\newblock 2019.
\newblock Funneled bayesian optimization for design, tuning and control of
  autonomous systems.
\newblock {\em IEEE Trans. Cybern.} 49(4):1489--1500.

\bibitem[\protect\citeauthoryear{Molina and
  Herrera}{2015}]{molina2015iterative}
Molina, D., and Herrera, F.
\newblock 2015.
\newblock Iterative hybridization of de with local search for the {CEC'2015}
  special session on large scale global optimization.
\newblock In {\em IEEE Congr. Evol. Comput.},  1974--1978.

\bibitem[\protect\citeauthoryear{Molina, LaTorre, and
  Herrera}{2018}]{molina2018iterative}
Molina, D.; LaTorre, A.; and Herrera, F.
\newblock 2018.
\newblock {SHADE} with iterative local search for large-scale global
  optimization.
\newblock In {\em IEEE Congr. Evol. Comput.}

\bibitem[\protect\citeauthoryear{Omidvar \bgroup et al\mbox.\egroup
  }{2010}]{omidvar2010random}
Omidvar, M.~N.; Li, X.; Yang, Z.; and Yao, X.
\newblock 2010.
\newblock Cooperative co-evolution for large scale optimization through more
  frequent random grouping.
\newblock In {\em IEEE Congr. Evol. Comput.}

\bibitem[\protect\citeauthoryear{Omidvar \bgroup et al\mbox.\egroup
  }{2017}]{omidvar2017faster}
Omidvar, M.~N.; Yang, M.; Mei, Y.; Li, X.; and Yao, X.
\newblock 2017.
\newblock {DG2}: A faster and more accurate differential grouping for
  large-scale black-box optimization.
\newblock {\em IEEE Trans. Evol. Comput.} 21(6):929--942.

\bibitem[\protect\citeauthoryear{Omidvar, Li, and
  Yao}{2010}]{omidvar2010cooperative}
Omidvar, M.~N.; Li, X.; and Yao, X.
\newblock 2010.
\newblock Cooperative co-evolution with {Delta} grouping for large scale
  non-separable function optimization.
\newblock In {\em IEEE Congr. Evol. Comput.}

\bibitem[\protect\citeauthoryear{Shahriari \bgroup et al\mbox.\egroup
  }{2016}]{shahriari2016human}
Shahriari, B.; Swersky, K.; Wang, Z.; Adams, R.~P.; and de~Freitas, N.
\newblock 2016.
\newblock Taking the human out of the loop: A review of bayesian optimization.
\newblock {\em Proceedings of the IEEE} 104(1):148--175.

\bibitem[\protect\citeauthoryear{Smith and
  Waterman}{1981}]{smith1981identification}
Smith, T.~F., and Waterman, M.~S.
\newblock 1981.
\newblock Identification of common molecular subsequences.
\newblock {\em Journal of Molecular Biology} 147(1):195--197.

\bibitem[\protect\citeauthoryear{Sun, Kirley, and Halgamuge}{2018}]{SUN2018Rec}
Sun, Y.; Kirley, M.; and Halgamuge, S.
\newblock 2018.
\newblock A recursive decomposition method for large scale continuous
  optimization.
\newblock {\em IEEE Trans. Evol. Comput.} 22(5):647--661.

\bibitem[\protect\citeauthoryear{Tanabe and
  Fukunaga}{2014}]{tanabe2014improving}
Tanabe, R., and Fukunaga, A.~S.
\newblock 2014.
\newblock Improving the search performance of shade using linear population
  size reduction.
\newblock In {\em IEEE Congress on Evolutionary Computation},  1658--1665.

\bibitem[\protect\citeauthoryear{Tseng and Chen}{2008}]{tseng2008multiple}
Tseng, L.-Y., and Chen, C.
\newblock 2008.
\newblock Multiple trajectory search for large scale global optimization.
\newblock In {\em IEEE Congr. Evol. Comput.},  3052--3059.

\bibitem[\protect\citeauthoryear{Yang, Tang, and Yao}{2008}]{yang2008large}
Yang, Z.; Tang, K.; and Yao, X.
\newblock 2008.
\newblock Large scale evolutionary optimization using cooperative coevolution.
\newblock {\em Inf. Sci.} 178(15):2985--2999.

\bibitem[\protect\citeauthoryear{Ye \bgroup et al\mbox.\egroup
  }{2014}]{ye2014hybrid}
Ye, S.; Dai, G.; Peng, L.; and Wang, M.
\newblock 2014.
\newblock A hybrid adaptive coevolutionary differential evolution algorithm for
  large-scale optimization.
\newblock In {\em IEEE Congr. Evol. Comput.},  1277--1284.

\end{thebibliography}

\clearpage
\section*{Appendix}
\subsection{Proof for Proposition \ref{prop:exploitation}}
\begin{proposition}[Exploration Bounds for \algoabbr{}]
Suppose every action is corresponded with at least one efficiency record within the sliding window of size $w$ and $\tau$ is the parameter for softmax decisions, the probability of taking the action with the not-the-highest mean normalized record (exploration) in the fragment of length $w$ satisfies the bounds
\begin{equation}
p_{\text{explore}} \in \left[\frac{e^{{1}/{\tau}}}{|A|-1+e^{{1}/{\tau}}}, \frac{|A|-1}{(|A|-1)\cdot e^{{1}/{\tau}}+e^{\frac{w-|A|}{\tau(w-|A|+1)}}}\right)
\end{equation}
\end{proposition}

\begin{proof}
After the normalization of the controller, the highest reward within the sliding window is $1$ and the lowest is $0$. We first stretch the reward stream fragment corresponding to the sliding window into a $|A| \times w$ rectangle by putting the rewards of each action on the corresponding row, as shown in Fig. \ref{figure:stretching}.
Making softmax decisions within the fragment does not care about the orders of rewards, thus an equivalence class of stretched fragments can be obtained by permutations on the fragment before stretching. Suppose that $a_1$ is the best performing action, we can show the upper bound can be obtained within the class equivalent to
\begin{figure}[!htbp]
  \centering
  \includegraphics[width=0.20\textwidth]{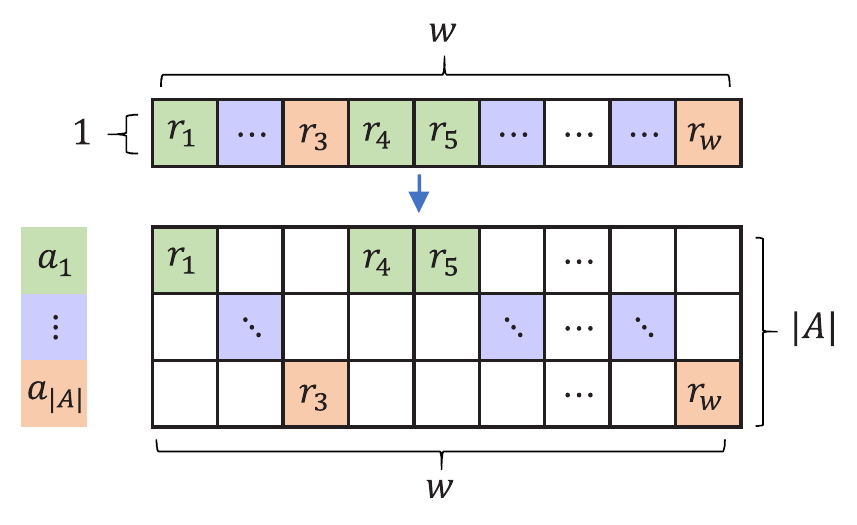}\\
  \setlength{\abovecaptionskip}{-0.2cm}
  \caption{Demonstration for stretching fragments.}
  \label{figure:stretching}
\end{figure}
\vspace{-0.2cm}

\begin{figure}[ht]
\vspace{0.3cm}
\centering
\includegraphics[width=0.2\textwidth]{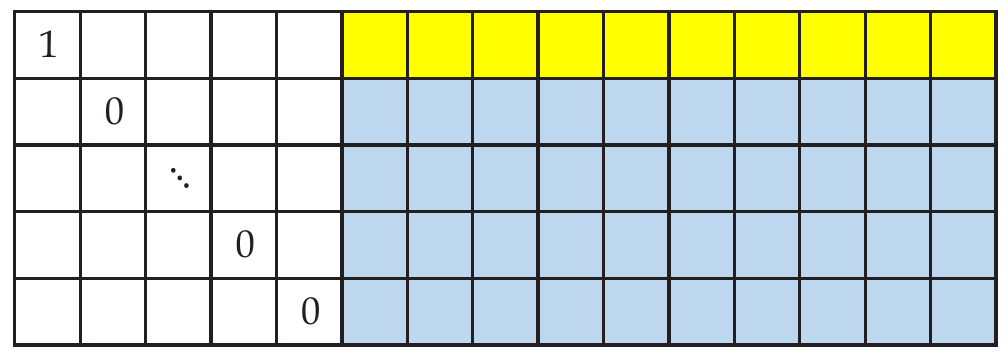}\\
\end{figure}
\vspace{-0.2cm}
\noindent where we can put $1$ in the yellow cells and $0$ in the cyan cells but at most one element in each column. Similarly, we can get the lower bound within the class equivalent to
\begin{figure}[!ht]
 \vspace{-0.1cm}
\centering
\includegraphics[width=0.2\textwidth]{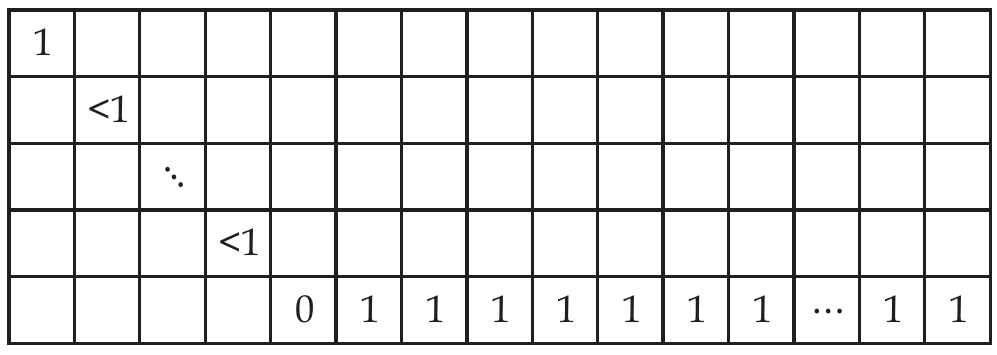}\\
\end{figure}
\vspace{-0.2cm}
\par
\noindent where ``$<1$'' represents normalized rewards that are infinitely close to $1$ but still less than $1$ and thus the lower bound cannot be reached.

The probability bounds for exploration can be derived directly by subtracting the exploitation bounds.
\end{proof}

\subsection{Principled Guidelines for Practical Use}
If we have a preferred interval for the probability of exploration or exploitation, we can inversely locate the potential combinations of the hyperparameters. For example, if there are totally $3$ heuristics to be articulated, $|A| = 3$, we can first constrain $w$ by $w \geq 4$ and simplify the exploitation bound according to Proposition 1.
Given a preferred exploitation probability interval, $(p_{min}, p_{max})$, we can solve $\tau$ using $p_{max}$ since only $\tau$ is involved in the upper-bound. Then, we use the solved $\tau$ and $p_{min}$ to get $w$.

\subsection{Details of Articulated Heuristics}

\begin{table}[htbp]
\centering
  \caption{Three Heuristics used for Experiments}
    \begin{tabular}{cp{18em}}
    \toprule
    \textbf{Name} & \multicolumn{1}{c}{LS} \\
    \textbf{Costs} & \multicolumn{1}{c}{$25 \times D$} \\
    \textbf{Details} & Local Search strategy used in MTS \cite{tseng2008multiple}, MOS \cite{latorre2012multiple}, resembling the trajectory search based algorithms. The same as the cofiguration in MOS: Initial step size is $0.2mean(\bm{u} - \bm{l})$, where $\bm{l}$ and $\bm{u}$ are the lower and upper box constraints respectively. Minimal step size is $1 \times {10}^{-15}$. \\
    \midrule
    \textbf{Name} & \multicolumn{1}{c}{CC} \\
    \textbf{Costs} & \multicolumn{1}{c}{$75 \times D$} \\
    \textbf{Details} & Cooperative Coevolution with random grouping \cite{omidvar2010random} and SaNSDE \cite{yang2008large} as optimizer, resembling the DECC family. A robust and classical configuration: The mean of NP is set to be $15$, group size is $50$ and $250$ generations is assigned for each group. \\
    \midrule
    \textbf{Name} & \multicolumn{1}{c}{GS} \\
    \textbf{Costs} & \multicolumn{1}{c}{$25 \times D$} \\
    \textbf{Details} & Global Search that applies SHADE \cite{tanabe2014improving} on all dimensions of the problem, resembling direct optimization strategies, \eg{} CSO \cite{cheng2015competitive}. The same configuration as SHADE-ILS \cite{molina2018iterative}: NP is set to be $50$, iterate for $D/2$ generations. \\
    \bottomrule
    \end{tabular}%
  \label{tab:heuristics}%
\end{table}%

\subsection{Hyperparameter Selection and Guideline}
We want to constrain the probability of exploitation to be at least $0.5$ and at most $0.99$. Using the bounds we have obtained, we can get three sets of solutions of $\langle \tau, w \rangle$: $\langle 1/5, 5 \rangle$, $\langle 1/6, 6 \rangle$, $\langle 1/7, 7 \rangle$ (considering only integer fractions). Then we test on the three combinations of hyperparameters, whose results are presented in Table \ref{tab:hyperparameters}. From the $t$-test results, the hyperparameter pairs $\langle 1/5, 5 \rangle$ and $\langle 1/6, 6 \rangle$ performed similarly well. In the experiments, we select the $\langle 1/5, 5 \rangle$ setting, since smaller $w$ leads to quicker adaptation.

\begin{table}[htbp]
\scriptsize
\setlength{\tabcolsep}{1pt}
  \centering
  \caption{Performance with Different Hyperparameters}
    \begin{tabular}{ccccccc}
    \toprule
    \multirow{2}[1]{*}{LSO13} & \multicolumn{2}{c}{$\langle 1/5, 5 \rangle$} & \multicolumn{2}{c}{$\langle 1/6, 6 \rangle$} & \multicolumn{2}{c}{$\langle 1/7, 7 \rangle$} \\
          & mean  & std   & mean  & std   & mean  & std \\
    \toprule
	$f_1$    & \cellcolor[rgb]{ .388,  .745,  .482}0.00E+00 & 0.00E+00 & \cellcolor[rgb]{ .388,  .745,  .482}0.00E+00 & 0.00E+00 & \cellcolor[rgb]{ .388,  .745,  .482}0.00E+00 & 0.00E+00 \\
    $f_2$    & \cellcolor[rgb]{ .388,  .745,  .482}8.69E+00 & 2.78E+00 & \cellcolor[rgb]{ 1,  .922,  .518}1.20E+01 & 2.54E+00 & \cellcolor[rgb]{ .973,  .412,  .42}1.28E+01 & 4.70E+00 \\
    $f_3$    & \cellcolor[rgb]{ 1,  .922,  .518}9.83E-13 & 5.27E-14 & \cellcolor[rgb]{ .388,  .745,  .482}8.51E-13 & 5.78E-14 & \cellcolor[rgb]{ .973,  .412,  .42}1.05E-12 & 5.35E-14 \\
    $f_4$    & \cellcolor[rgb]{ .973,  .412,  .42}6.98E+08 & 2.51E+08 & \cellcolor[rgb]{ .388,  .745,  .482}6.04E+08 & 2.02E+08 & \cellcolor[rgb]{ 1,  .922,  .518}6.27E+08 & 2.99E+08 \\
    $f_5$    & \cellcolor[rgb]{ .388,  .745,  .482}2.68E+06 & 4.38E+05 & \cellcolor[rgb]{ .973,  .412,  .42}2.72E+08 & 5.12E+05 & \cellcolor[rgb]{ 1,  .922,  .518}2.70E+06 & 6.04E+05 \\
    $f_6$    & \cellcolor[rgb]{ .388,  .745,  .482}4.44E+04 & 3.42E+04 & \cellcolor[rgb]{ 1,  .922,  .518}7.14E+04 & 1.92E+04 & \cellcolor[rgb]{ .973,  .412,  .42}9.48E+04 & 1.87E+04 \\
    $f_7$    & \cellcolor[rgb]{ 1,  .922,  .518}1.58E+05 & 3.68E+04 & \cellcolor[rgb]{ .388,  .745,  .482}1.34E+05 & 5.02E+04 & \cellcolor[rgb]{ .973,  .412,  .42}2.46E+05 & 7.81E+04 \\
    $f_8$    & \cellcolor[rgb]{ .388,  .745,  .482}1.24E+11 & 9.29E+10 & \cellcolor[rgb]{ 1,  .922,  .518}2.61E+11 & 1.51E+11 & \cellcolor[rgb]{ .973,  .412,  .42}7.16E+11 & 5.66E+11 \\
    $f_9$    & \cellcolor[rgb]{ 1,  .922,  .518}2.36E+08 & 2.61E+07 & \cellcolor[rgb]{ .388,  .745,  .482}2.25E+08 & 4.38E+07 & \cellcolor[rgb]{ .973,  .412,  .42}2.52E+08 & 4.64E+07 \\
    $f_{10}$   & \cellcolor[rgb]{ .388,  .745,  .482}7.64E+05 & 5.90E+05 & \cellcolor[rgb]{ .973,  .412,  .42}8.18E+05 & 5.79E+05 & \cellcolor[rgb]{ 1,  .922,  .518}8.00E+05 & 2.63E+05 \\
    $f_{11}$   & \cellcolor[rgb]{ 1,  .922,  .518}3.33E+07 & 1.08E+07 & \cellcolor[rgb]{ .388,  .745,  .482}2.76E+07 & 8.55E+06 & \cellcolor[rgb]{ .973,  .412,  .42}4.45E+07 & 3.48E+07 \\
    $f_{12}$   & \cellcolor[rgb]{ 1,  .922,  .518}6.27E+02 & 2.11E+02 & \cellcolor[rgb]{ .973,  .412,  .42}6.69E+02 & 2.09E+02 & \cellcolor[rgb]{ .388,  .745,  .482}4.32E+02 & 2.57E+02 \\
    $f_{13}$   & \cellcolor[rgb]{ .973,  .412,  .42}1.14E+07 & 2.20E+06 & \cellcolor[rgb]{ 1,  .922,  .518}1.07E+07 & 2.24E+06 & \cellcolor[rgb]{ .388,  .745,  .482}1.05E+07 & 1.13E+07 \\
    $f_{14}$   & \cellcolor[rgb]{ .973,  .412,  .42}4.35E+07 & 6.85E+06 & \cellcolor[rgb]{ .388,  .745,  .482}4.17E+07 & 7.80E+06 & \cellcolor[rgb]{ 1,  .922,  .518}4.29E+07 & 1.00E+07 \\
    $f_{15}$   & \cellcolor[rgb]{ 1,  .922,  .518}3.66E+06 & 2.32E+05 & \cellcolor[rgb]{ .388,  .745,  .482}3.63E+06 & 2.01E+06 & \cellcolor[rgb]{ .973,  .412,  .42}4.25E+06 & 9.17E+05 \\
    t-test & \multicolumn{2}{c}{$</\approx/>$} & \multicolumn{2}{c}{3/9/3} & \multicolumn{2}{c}{6/8/1} \\
    \bottomrule
	\multicolumn{7}{m{0.37\textwidth}}{\tiny Color indicators are added for each test case. The greener, the better performance.}\\
    \end{tabular}%
  \label{tab:hyperparameters}%
\end{table}%

\subsection{Optimization Curves for Reproduced Experiments}
We present a representative set of optimization curves for the scalability tests on $f_4$ of LSO08 in Fig. \ref{fig:CCLSO08F4}. The curves only include those that we are able to implement or reproduce.

\begin{figure}[!htbp]
\centering

\subfloat[$D = 1000$]{
\captionsetup{justification = centering}
\includegraphics[width=0.22\textwidth]{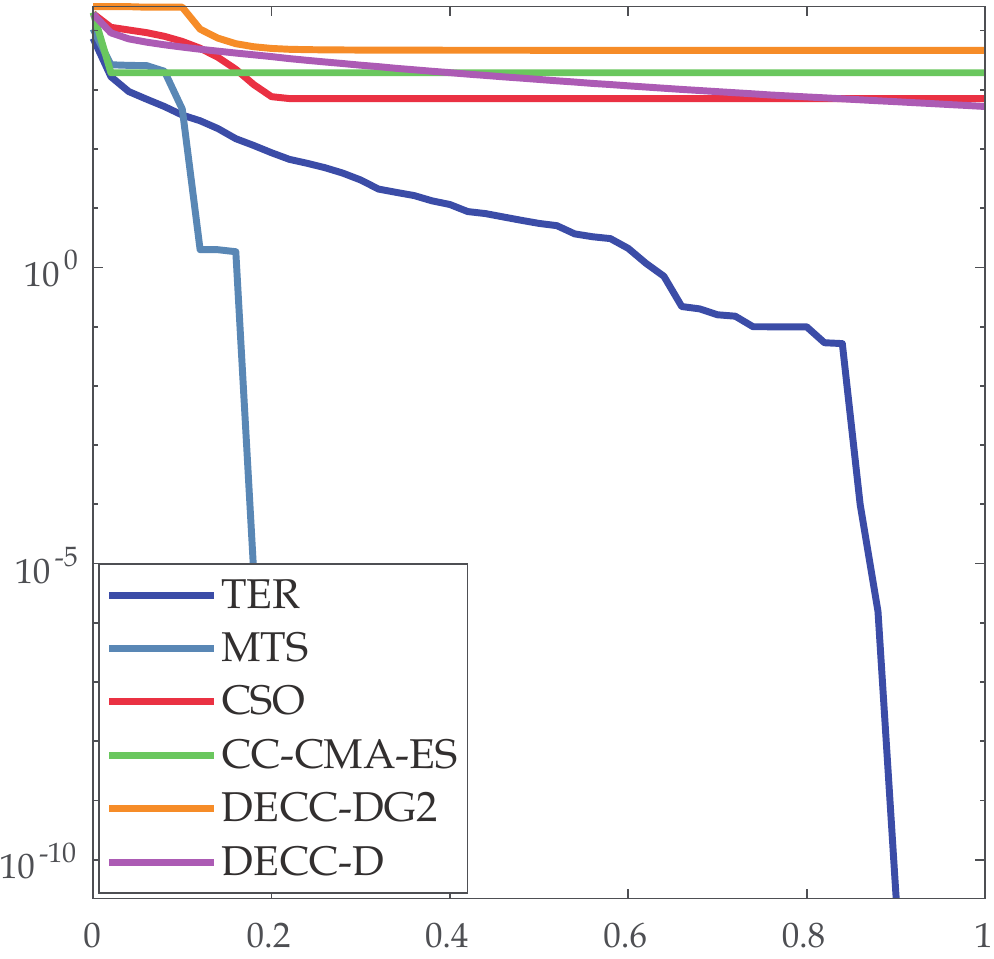}}
\hfill
\subfloat[$D = 2500$]{
\captionsetup{justification = centering}
\includegraphics[width=0.22\textwidth]{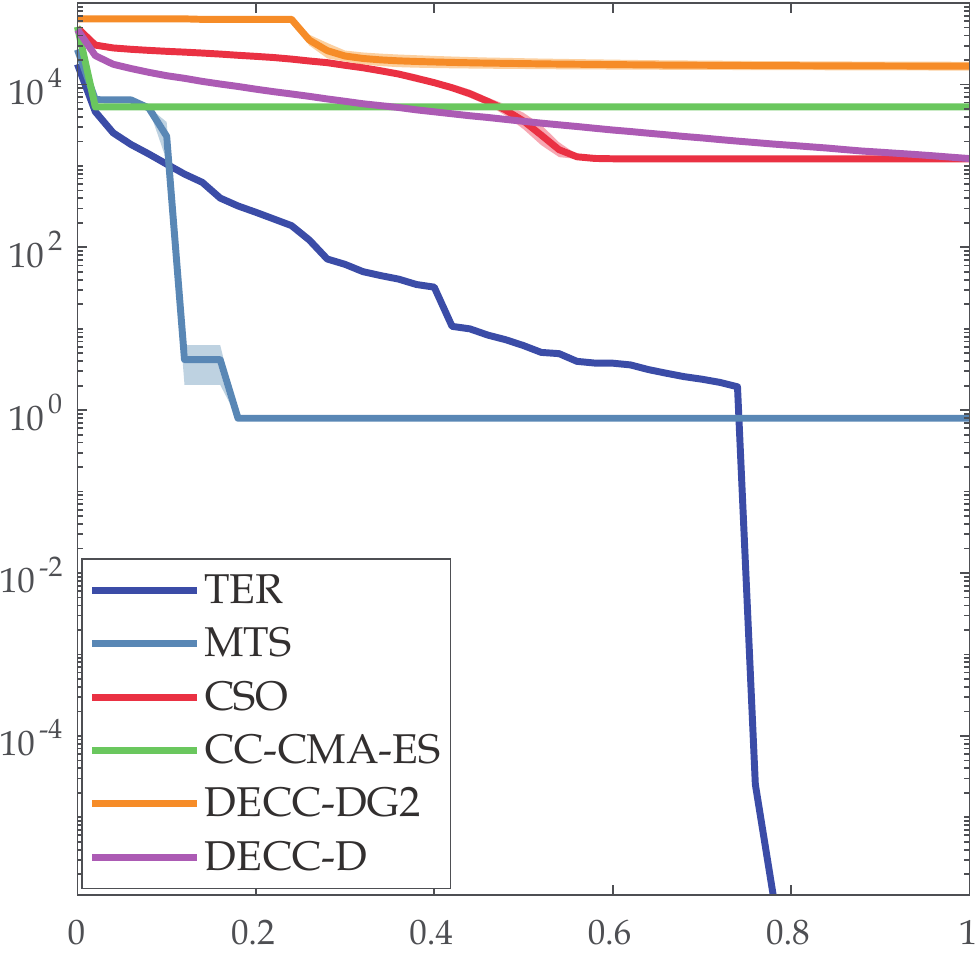}}

\subfloat[$D = 5000$]{
\captionsetup{justification = centering}
\includegraphics[width=0.22\textwidth]{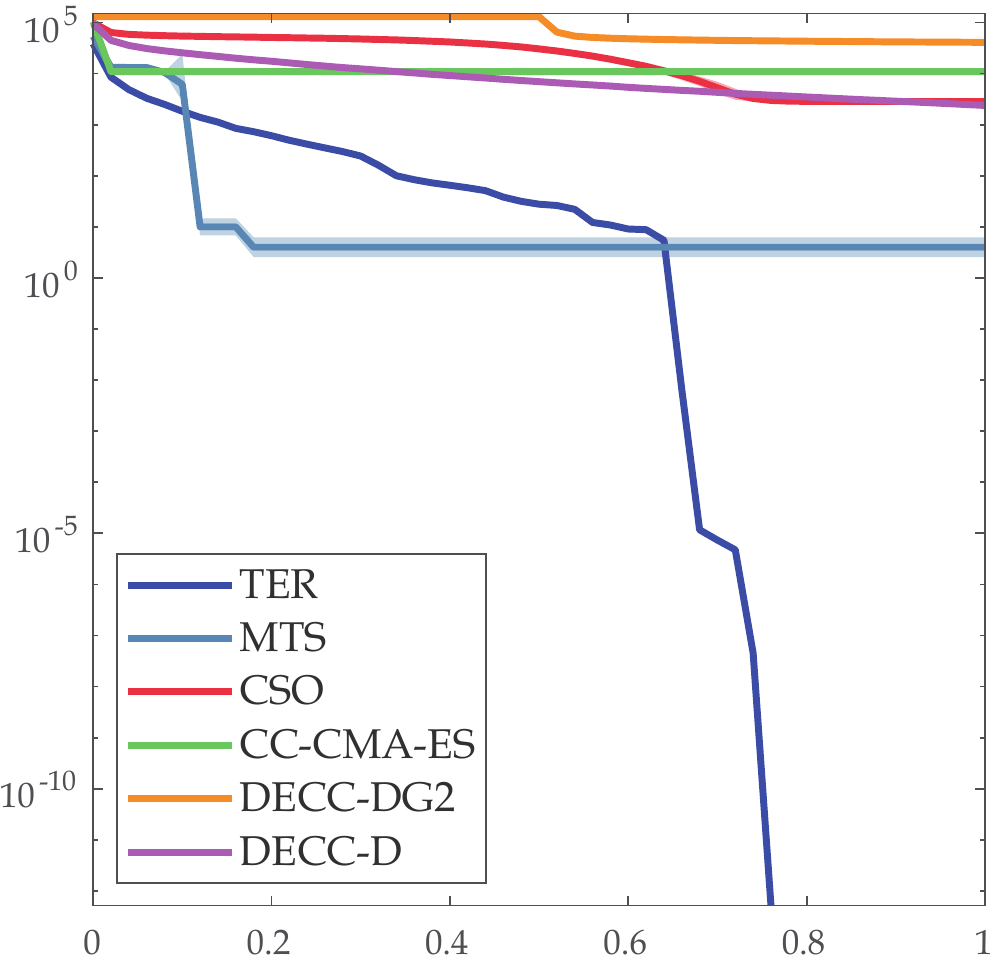}}
\hfill
\subfloat[$D = 10000$]{
\captionsetup{justification = centering}
\includegraphics[width=0.22\textwidth]{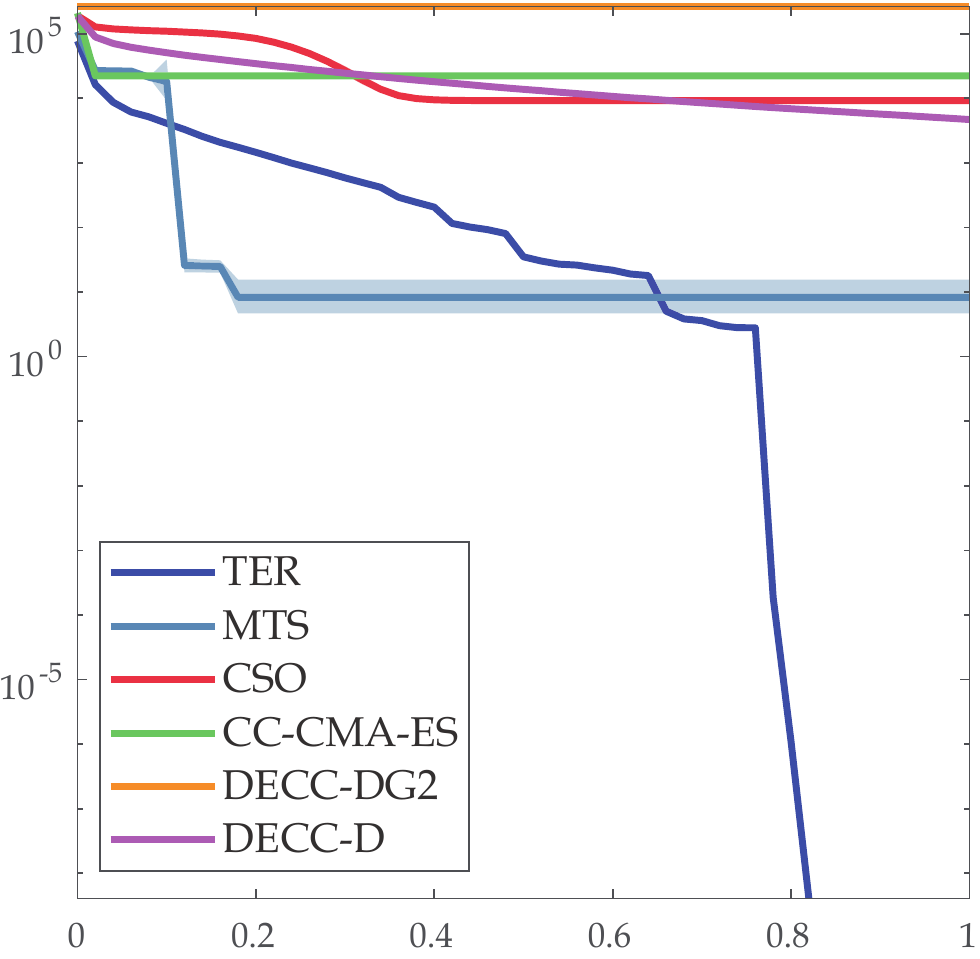}}

\caption{$f_4$ of LSO08.}
\label{fig:CCLSO08F4}
\end{figure}

\subsection{Optimization Curves for Baseline Comparison}
We present some representative sets of optimization curves for the baseline comparison in Fig. \ref{fig:DCLSO}.

\begin{figure}[htbp]
\centering

\subfloat[$f_7$ of LSO13]{
\captionsetup{justification = centering}
\includegraphics[width=0.28\textwidth]{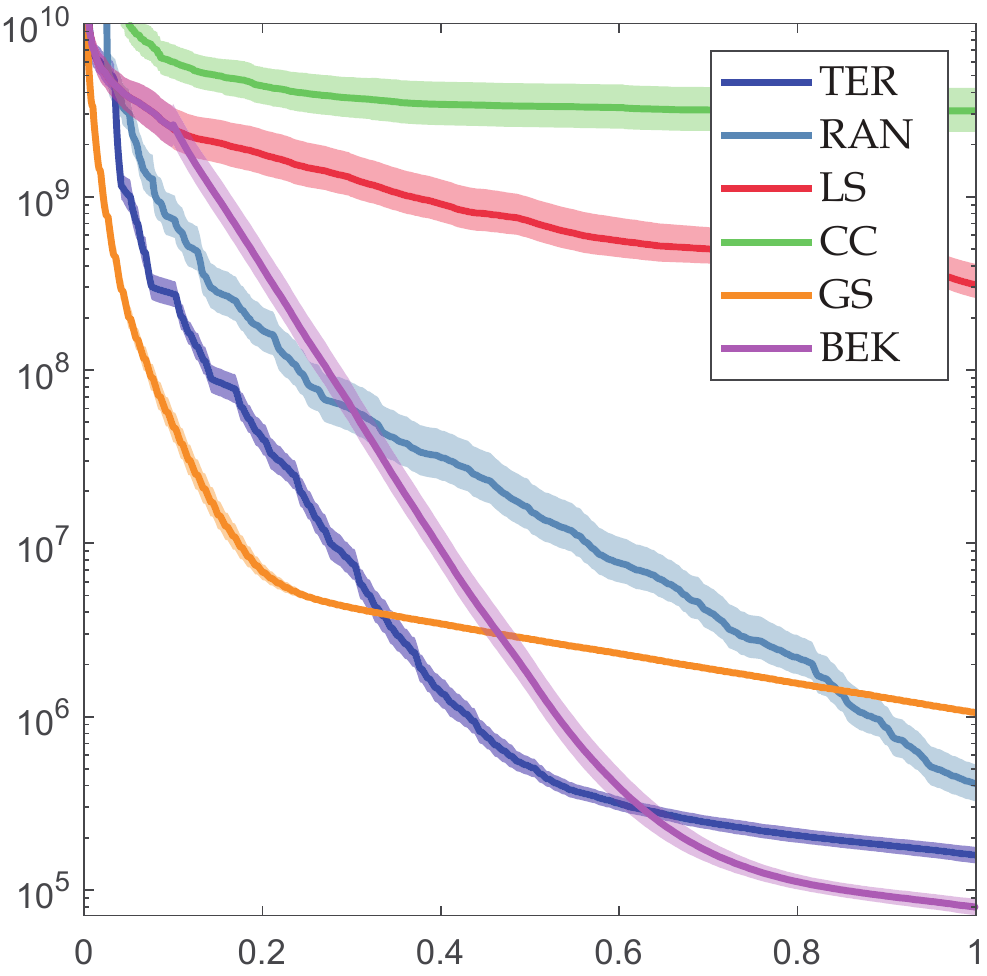}}

\subfloat[$f_8$ of LSO13]{
\captionsetup{justification = centering}
\includegraphics[width=0.28\textwidth]{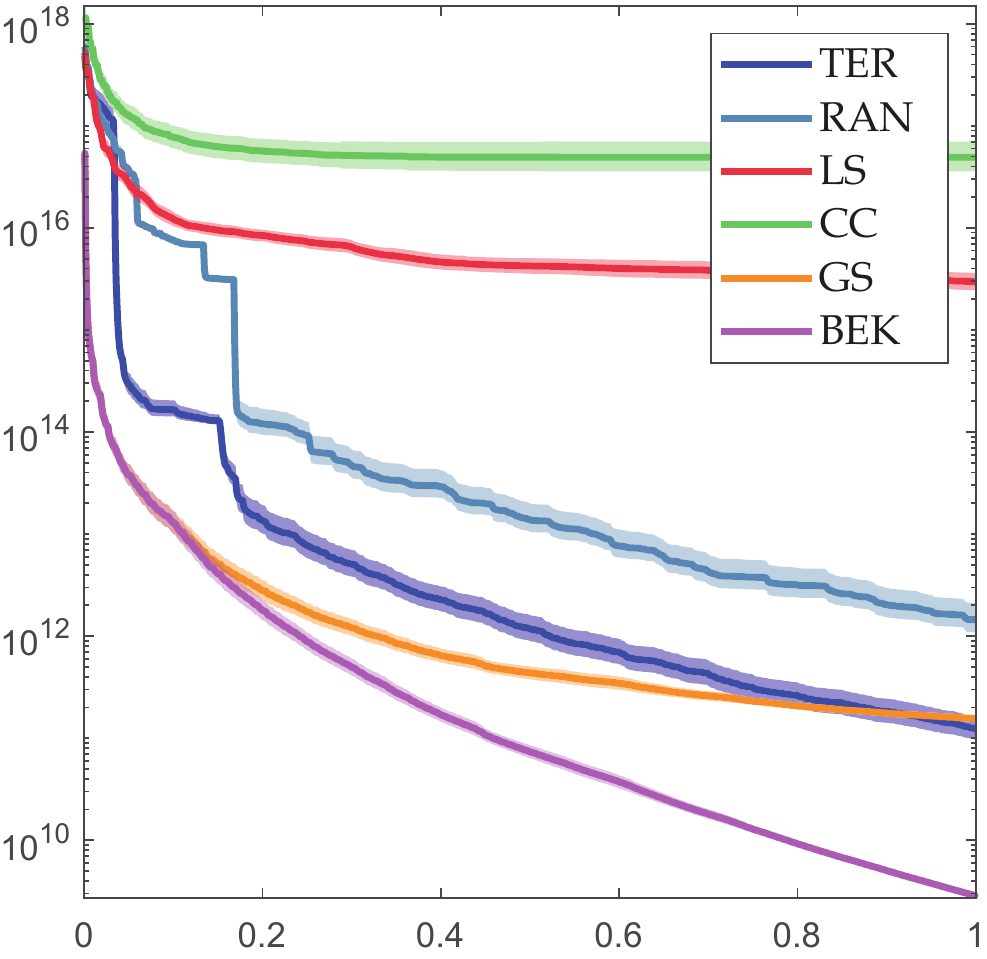}}

\subfloat[$f_15$ of LSO13]{
\captionsetup{justification = centering}
\includegraphics[width=0.28\textwidth]{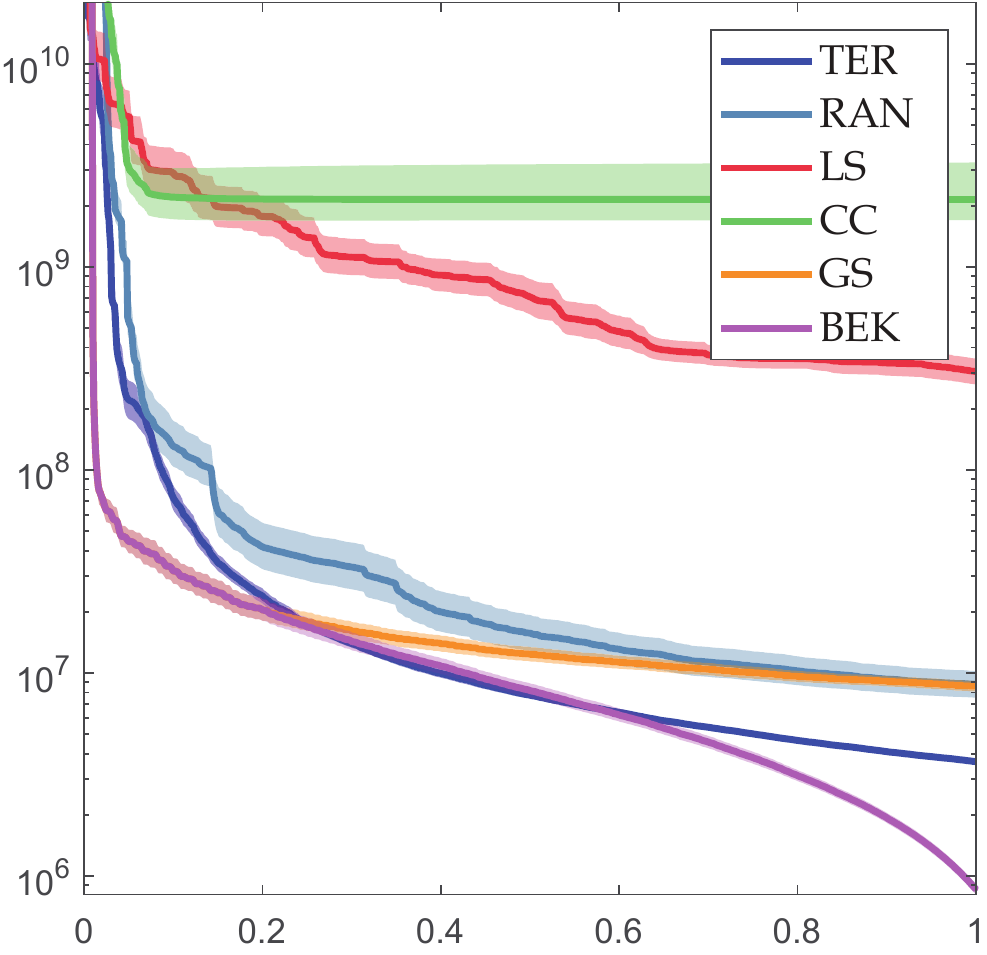}}

\caption{Selected curves for the baseline comparison.}
\label{fig:DCLSO}
\end{figure}

\end{document}